\documentclass[journal]{IEEEtran}

\usepackage{tikz-cd}
\usepackage{arydshln}
\usepackage{signalpreamble}

\AtBeginDocument{%
  \normalfont
}

\crefname{bound}{bound}{bounds}

\myexternaldocument{main-supp}

\newcommand{\diam}{\mathrm{diam}\,}

\newcommand{\Vsub}{U}
\newcommand{\Vsubc}{U\setcomp}
\newcommand{\Gsub}{G_{\Vsub}}

\newcommand{\Psub}{\bP_{\Vsub}}

\NewDocumentCommand\Ssub{ o }{%
	\IfValueTF{#1}%
	{\bS_{#1, \Vsub}}%
	{\bS_{\Vsub}}%
}
\newcommand{\Fsub}{\bF_{\Vsub}}
\newcommand{\hFsub}{\widehat{\bF}_{\Vsub}\tc{n}}
\newcommand{\Lsub}{\bL_{\Vsub}}

\newcommand{\Msub}{\bM_{\Vsub\Vsub}}
\newcommand{\tbL}{\widetilde{\mathbf{L}}}

\newcommand{\cut}{\mathrm{cut}}

\newcommand{\lambdamax}[1]{\lambda_{\max}\left(#1\right)}
\newcommand{\lambdamin}[1]{\lambda_{\min}\left(#1\right)}

\newacronym{GSP}{GSP}{graph signal processing}
\newacronym{MSE}{MSE}{mean squared error}
\newacronym{GSO}{GSO}{graph shift operator}
\newacronym{DSSO}{DSSO}{distance-based subgraph shift operator}
\newacronym{GFT}{GFT}{graph Fourier transform}
\newacronym{GNN}{GNN}{graph neural network}
\newacronym{DLP}{DLP}{distance-based Laplacian polynomial}
\newacronym{SFFA}{SFFA}{subgraph filter free algebra}
\newacronym{ST-SFFA}{ST-SFFA}{spatiotemporal SFFA}
\newacronym{SHLA}{SHLA}{symmetric hop Laplacian algebra}
\newacronym{SCA}{SCA}{subgraph cover algebra}
\newacronym{NESCA}{NESCA}{nested ensemble subgraph cover algebra}
\newacronym{SDO}{SDO}{sheaf diffusion operator}
\newacronym{TSO}{TSO}{temporal shift operator}
\newacronym{SSI}{SSI}{semi-shift invariant}
\newacronym{SFL}{SFL}{subgraph filter learning}
\newacronym{LSI}{LSI}{linear shift-invariant}
\newacronym{ST-SFL}{ST-SFL}{spatiotemporal subgraph filter learning}
\newacronym{GMRF}{GMRF}{Gaussian Markov random field}
\newacronym{WSS}{WSS}{wide-sense stationary}
\newacronym{LBFGS}{LBFGS}{limited-memory BFGS}
\newacronym{DFT}{DFT}{discrete Fourier transform}
\newacronym{AR}{AR}{autoregressive}
\newacronym{RI}{RI}{relative improvement}

\begin{document}

\title{Filter Learning for Subgraphs: Algebras and Performance Risk Bounds}

\author{
Purui Zhang,
Feng Ji,
Yanan Zhao,
Bihan Wen, and 
Wee Peng Tay
}

\maketitle

\begin{abstract}
Graph signal processing tasks that leverage spectral information typically assume access to the complete graph topology, which is often unavailable in practice. We propose a systematic framework for subgraph filter learning (SFL), where subgraph-supported operators approximate ambient graph filters under partial observations. We formulate SFL as a statistical learning problem in which optimal subgraph operators are inherently data-dependent. To address the difficulty of directly estimating such operators, we develop a subgraph filter algebra based on distance-aware Laplacian constructions, defining a structured and controllable class of filters for effective approximation. We further establish performance risk bounds under the least squares loss, quantifying how well the learned operator approximates the restricted ambient mapping. Experiments real-world datasets show that, for SFL tasks, the proposed algebraic models consistently outperform polynomial filters, distribution-agnostic operators, and direct numerical filter learning baselines that attempt to recover the underlying structure from data.
\end{abstract}

\section{Introduction}\label{sec:intro}

Since its inception, \gls{GSP} has encompassed a diverse range of research areas, including the analysis of graph-structured data \cite{Shu13,Ort18,MatSeg19}, statistical and spectral graph analysis \cite{Gir:C15,PerVan:J17, JiaTayEld24, JiaGolJi25}, the inference of graph topology \cite{Don16,Egi17,ZhuJin:C17}, and advancements in graph neural networks \cite{Scar09,kipf17,JiZhaLee25}. At the heart of \gls{GSP} lies the \gls{GSO}, which serves as a fundamental component directly linked to the graph's topology. The \gls{GSO} plays a pivotal role in enabling key operations such as filtering \cite{Chen14, JiTayOrt23}, signal recovery \cite{Chen15,SegMar:J16,QiuMao:J17,WangChen:C15,CaoWang:C25}, and the computation of the \gls{GFT} \cite{Sand13gft, JiTay19}.

Conventional \gls{GSP} tasks, including filtering, spectral analysis, and stationarity analysis\cite{JiTay19,JiaTay22a}, typically assume access to the full graph topology and its associated \gls{GSO}. In many practical settings, however, only a partial view of the graph is available. A stakeholder who owns the graph may grant a third party access to only a subset of nodes because of privacy, computational, or operational constraints; see \cref{fig:eg_subgraph}. The third party seeks to estimate an output signal by filtering over the accessible subgraph, but it cannot directly observe the full graph topology or the complete input graph signal. Instead, the stakeholder provides a subgraph filter (or \gls{GSO}) defined solely on the accessible nodes.\footnote{Alternatively, the stakeholder could provide filtered data. However, this approach incurs recurring service overhead, as the stakeholder must respond to each processing request. By contrast, providing a reusable subgraph filter is a one-time information transfer that avoids this overhead.} The quality of a subgraph filter is evaluated by comparing two outputs: (i) the ambient graph filter applied to the full signal, followed by restriction to the node subset; and (ii) the subgraph filter applied directly to the restricted signal. A sufficiently small discrepancy between these outputs indicates that the subgraph filter is useful for the third party.

\begin{figure}[!t]
\begin{center}
\includegraphics[width=0.98\columnwidth]{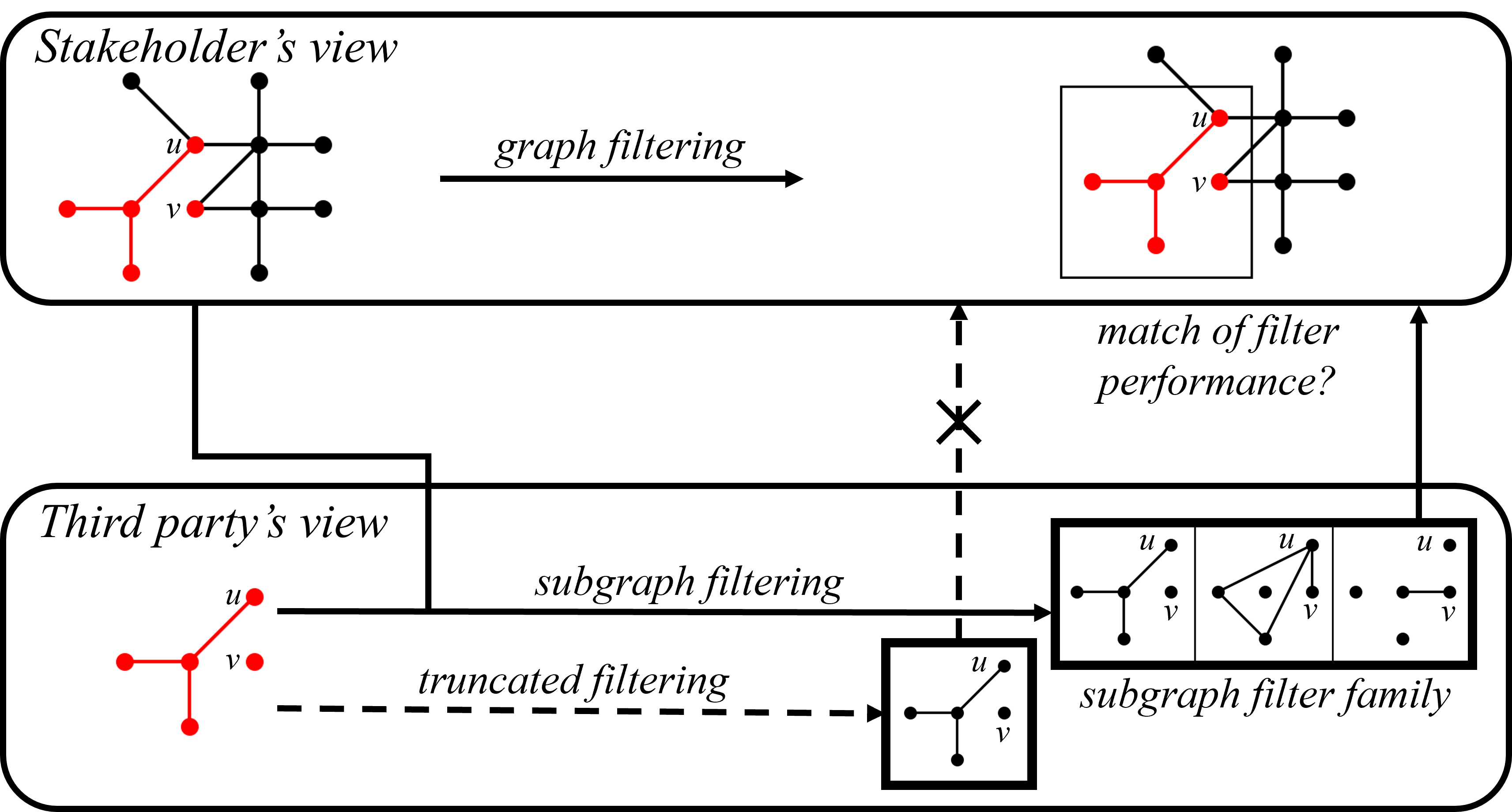}
\end{center}
\vspace{-5mm}
\caption{Illustration of subgraph filtering. Note that the third party has no knowledge of nodes and edges beyond the subgraph structure. In the stakeholder's graph filter with at least 2-hop connectivity, information from node $v$ contributes to the output at node $u$. However, in the truncated filter, the contribution from node $v$ to node $u$ is lost. In contrast, using filters from a subgraph filter family can effectively capture the influence of node $v$ on node $u$, resulting in improved performance.}
\vspace*{-3pt}
\label{fig:eg_subgraph}
\end{figure}

A naive approach is a truncated filter formed by extracting the principal submatrix of the ambient filter (represented as a matrix) corresponding to the subgraph nodes. Although this construction preserves message passing within the subgraph, it removes all interactions between subgraph nodes and external nodes, leading to systematic information loss. Consequently, truncated filters often incur substantial approximation error. This limitation motivates the development of a more expressive class of subgraph filters that can better approximate the restricted ambient filtering operation without requiring access to the full graph structure.

We study \gls{SFL}, which seeks to learn a subgraph-supported filter for a specified node subset from ambient graph data. The learned filter is deployed only on the accessible subgraph and approximates the restricted action of the ambient filtering operation, without revealing the ambient filter or the graph structure outside the subset. To the best of our knowledge, no comprehensive study has systematically formulated \gls{SFL} and derived its error analysis with respect to both data and graph structure.
Our main contributions are summarized as follows:
\begin{itemize}
    \item We formulate the problem of \gls{SFL}, which studies how to construct subgraph-supported filters that approximate the restricted action of an ambient graph transformation. We propose a $k$-hop, distance-based subgraph filter space based on matrix algebras for \gls{SFL}. 
    
    \item We provide a finite-sample excess-risk analysis of ridge-regularized least-squares \gls{SFL} over a $k$-hop, distance-based subgraph filter space under sub-Gaussian graph-\gls{WSS} inputs. The bound decomposes prediction error into (i) irreducible loss from subgraph-only observations, (ii) approximation error from excluding interactions beyond $k$ hops in the ambient graph, (iii) regularization error, and (iv) filter-bank dimension effects. We also show how graph-distance structure controls the trade-off between model expressiveness and finite-sample estimation error.

    \item We evaluate the \gls{SFL} framework on filtering, signal reconstruction, and prediction tasks using both synthetic and real-world datasets. The results verify consistent improvements over standard polynomial filtering methods and statistical baselines.
    
\end{itemize}

A preliminary version of this work was presented in \cite{ZhJi:C25}. Compared with the preliminary version, this paper recasts \gls{SFL} in a statistical signal processing framework, generalizes distance-based constructions from normalized Laplacians to broader types of shift operators on subgraphs, introduces a finite-sample excess-risk analysis for ridge-regularized \gls{SFL}, and significantly expands experiments to filtering, reconstruction, and prediction with additional hyperparameter and ablation studies.

\subsection{Related Work}\label{subsec:related-work}
\subsubsection{Signal processing on substructures}

Several prior works have studied \gls{GSP} on graph substructures, including local-set-based graph signal reconstruction~\cite{WangLiu:J15}, graph signal decimation~\cite{TreBor:J16}, and spectral graph reduction~\cite{And:J19}. In \glspl{GNN}, SplitGNN~\cite{WuYao:C23} partitions the original graph and applies spectral graph filters to address heterophily in fraud detection, while STMS-GCN~\cite{WangGuo:C25} employs multiple learnable subgraphs to capture spatio-temporal motion patterns. These works use graph substructures for reconstruction, reduction, or representation learning, whereas \gls{SFL} characterizes the approximation of ambient-graph transformations on prescribed subgraphs.

\subsubsection{Graph filtering and filter-bank design}

Polynomial graph filters for a prescribed \gls{GSO} have been widely studied, including optimal designs for implementing or approximating linear network operators~\cite{SegMar:J17}. Extensions beyond a single \gls{GSO} include parallel filters associated with different \glspl{GSO}~\cite{HuaRi:J19}, polynomial filters of multiple commutative shifts~\cite{EmiChe:J22}, and graph-based classification with multiple shift matrices based on features~\cite{FanTep:J22}. Expressiveness has also been improved by relaxing node-invariant parameter sharing through node-variant~\cite{SegMar:C16}, edge-varying~\cite{IsuGam:J22}, and partition-wise~\cite{LiYang:C25} graph filters, which assign different weights to different nodes, edges, and node partitions, respectively. Although applicable to induced subgraphs, these methods do not explicitly address the ambient-subgraph relationship central to \gls{SFL}, which requires filter banks tailored to approximating ambient transformations on prescribed subgraphs.

\subsubsection{Stochastic graph signal estimation}

Stationary \gls{GSP} models graph random processes through graph-dependent second-order statistics~\cite{Gir:C15,MarSeg:J17,PerVan:J17}, supporting mean-square-optimal graph Wiener filters and Bayesian estimation under prescribed stochastic models~\cite{IsuDi:C18,HarTan:J23,KroRou:J22,FukHar:C25}. Random-design regression provides an unstructured finite-sample counterpart. Specifically, the paper \cite{Hsu:J12} analyzes ordinary least squares and ridge regression, while \cite{MouRos:J22} presented a concise excess-risk analysis of ridge regression in the same setting. Unlike full-graph stochastic estimators and unstructured random-design analyses, \gls{SFL} learns subgraph-supported filters from graph-signal pairs and quantifies how intrinsic prediction and filter-bank approximation errors depend on graph distances, subgraph boundary interactions, or the way a prescribed subgraph is embedded in the ambient graph.

\emph{Organization of the paper.} In \cref{sec:filter-learning}, we formalize the \gls{SFL} problem and motivate the need for new filter-bank designs. In \cref{sec:subgraph-filter-algebra}, we present a filter-algebra architecture and introduce distance-based subgraph operators that integrate with this architecture. In \cref{sec:error-bound}, we provide a theoretical analysis of subgraph-filter approximation error, characterizing the intrinsic limits of \gls{SFL} and informing the design of subgraph filter banks. Finally, in \cref{sec:experiments}, we evaluate the effectiveness of different subgraph filter-bank types across diverse \gls{SFL} tasks.

\emph{Notations.} Let $\bbR$ denote the set of real numbers and $\bbN$ the set of natural numbers. Matrices and vectors are represented by boldface letters. For a set of signals $\set{\bb_i\in\bbR^N}_{i=1}^T$, the design matrix $[\bb_i]_{i=1}^T\in\bbR^{N\times T}$ is formed by stacking the column vectors $\bb_i$ along the rows.
Let $\Vsub \subset \set{1,\dots,N}$ be a subset of indices. We define the partition of a matrix $\bS\in\bbR
^{N\times N}$ with respect to $\Vsub$ as (with reordering of rows and columns if necessary) $\bS = \smat{\bS_{\Vsub\Vsub}&\bS_{\Vsub\Vsubc} \\ \bS_{\Vsubc\Vsub}&\bS_{\Vsubc\Vsubc}}$.

Given a set of linear operators $\calS$, $\spn\calS$ denotes the span of $\calS$, i.e., the set of linear combinations of elements in $\calS$. 

\section{The Subgraph Filter Learning Problem}\label{sec:filter-learning}

In this section, we first formalize the \gls{SFL} problem, which is a statistical learning problem that aims to find a subgraph-supported operator that approximates the restricted ambient mapping. Then, we present the empirical formulation of the problem, which is more practical for real-world applications. Finally, we provide examples as illustrations.

\subsection{Problem Formulation}\label{subsec:prob-form}

Let $G=(V,E)$ be an undirected, unweighted graph with vertex set $V$, $|V|=N$, and edge set $E$. The adjacency matrix $\bA\in\set{0,1}^{|V|\times|V|}$ is defined by $\bA_{ij}=1$ if and only if $(v_i, v_j)\in E$, and $\bA_{ij}=0$ otherwise. The degree matrix $\bD=\diag\set{\deg(v_i)}_{i=1}^{|V|}$ is a diagonal matrix, where $\deg(v_i)$ denotes the degree of vertex $v_i$ in $G$. The Laplacian is defined as $\bL=\bD-\bA$, and the normalized Laplacian as $\tbL=\bI-\bD^{-1/2}\bA\bD^{-1/2}$. 

Suppose the graph $G$ is owned by a stakeholder. Associated with the graph is a random input graph signal $\bx\in \bbR^{|V|}$ where $\E\norm{\bx}_2<\infty$. The stakeholder grants access to a subset of nodes $\Vsub \subseteq V$, with $|\Vsub|=N_0$, to a third party. The $\Vsub$-induced subgraph is defined as $G_{\Vsub}\coloneq(\Vsub, \set{(u,v)\in E \given u, v\in \Vsub})$. The restriction of $\bx$ to $\Vsub$ is given by $\bx_{\Vsub}\coloneq\Psub\bx\in\bbR^{|\Vsub|}$, where $\Psub\in \set{0,1}^{|\Vsub|\times |V|}$ is the \emph{selection matrix}, satisfying $\Psub\Psub\T=\bI_{\Vsub}$, the $|\Vsub|\times |\Vsub|$ identity matrix.

An output signal is modeled as $\by=f(\bx)+\bw \in \bbR^N$, where $\E\norm{f(\bx)}_2<\infty$, and $\bw$ is a random noise vector, which is uncorrelated to $\bx$. It satisfies $\E\bw=\bzero$ and $\E\norm{\bw}_2<\infty$.

A third party has no knowledge of $f(\cdot)$ or the full ambient graph $G$. Instead, it has access only to the $\Vsub$-induced subgraph $\Gsub$ and the restricted signals $\bx_{\Vsub}$. To estimate the restricted output signal $\by_{\Vsub}=\Psub\by$, the stakeholder provides the third party with a subgraph filter $\Fsub$ supported on $\Gsub$ (see \cref{fig:comm}). The objective is to minimize the \emph{performance risk}:
\begin{align}\label{eq:p-loss}
\calR_{\calL}(\Fsub)
=\E[\calL\parens*{\by_{\Vsub},\Fsub\bx_{\Vsub}}]
=\E[\calL\parens*{\Psub \by,\Fsub\Psub\bx}],
\end{align}
where $\calL\parens{\cdot,\cdot}$ is a nonnegative loss function. 

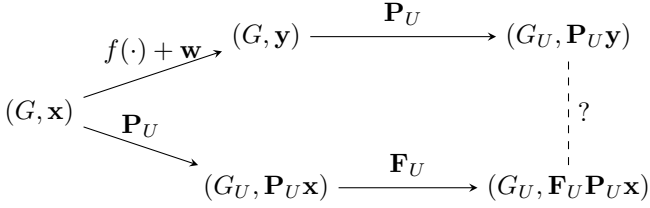
\begin{figure}[!htb]
\centering

\begin{tikzpicture}[>=stealth, node distance=1cm]
\node (G) at (-4,0) {$(G, \bx)$};
\node (Gz) at (-1,1) {$(G, \by)$};
\node (Gy) at (-1,-1) {$(\Gsub, \Psub\bx)$};
\node (Gzy) at (3,-1) {$(\Gsub, \Fsub\Psub\bx)$};
\node (Gzb) at (3,1) {$(\Gsub, \Psub\by)$};

\draw[->] (G) -- (Gz) node[midway, above] {$f(\cdot)+\bw$};
\draw[->] (G) -- (Gy) node[midway, above] {$\Psub$};
\draw[->] (Gy) -- (Gzy) node[midway, above] {$\mathstrut\Fsub$};
\draw[->] (Gz) -- (Gzb) node[midway, above] {$\mathstrut\Psub$};

\draw[-, dashed] (Gzy) -- (Gzb) node[midway, right] {$?$};

\end{tikzpicture}
\caption{The subgraph filter learning framework. The upper path represents the stakeholder's processing pipeline, and the lower path represents the third party's processing pipeline. The goal of \gls{SFL} is to design subgraph filter $\Fsub$ such that the two pipelines produce consistent outputs on the subgraph.}
\label{fig:comm}
\end{figure}

The subgraph filter $\Fsub$ is constructed as a linear combination of basis elements from a specified subgraph filter bank $\calS(\Gsub)=\set{\bB_j\in\bbR^{N_0\times N_0}}_{j=1}^m$, such that
\begin{align}\label{eq:subgraph-filter}
\Fsub=\sum_{j=1}^m\theta_j\bB_j,
\end{align}
where $\bB_j\in\calS(\Gsub)$, and the filter bank size $m$ is constrained by the budget
\begin{align}\label{eq:budget}
    m\leq \alpha N_0^\beta,\ \mathrm{with}\  \alpha>0,\ 0<\beta\leq 2.
\end{align}

The performance loss in \cref{eq:p-loss} quantifies the expected discrepancy between the restricted ambient transformation $\Psub \by$ and the output $\Fsub\Psub\bx$ generated by the subgraph filter. This metric serves as a measure of the predictive accuracy of the subgraph filter. The optimization problem in \gls{SFL} is inherently challenging (under arbitrary loss functions) due to two primary factors: (i) the fundamental approximation arising from the restricted observation $\bx_{\Vsub}$, even when the filter space is unconstrained; and (ii) the additional approximation error introduced by the constrained dimensionality of the filter space $\calS(\Gsub)$.

In the context of the scenario in \cref{sec:intro}, the stakeholder's objective is to construct a subgraph filter $\Fsub$ such that $\Fsub\Psub\bx \approx \Psub \by$ in expected loss, which can then be transmitted to the third party. In scenarios where the mapping $f(\cdot)$ is unknown, the stakeholder relies on training data comprising input-output signal pairs $(\bx,\by)$ to learn $\Fsub$ by minimizing the empirical performance loss.
As demonstrated in \cref{subsec:empirical}, even when the data distribution is fully known, the oracle subgraph filter generally differs from that based on the naive induced subgraph shift operator. In particular, it includes a boundary-correction term arising from the unobserved complement $\Vsub^c$. This observation motivates the construction of an appropriate subgraph filter bank $\calS(\Gsub)$ to accurately approximate the target restricted transformation.

The choice of the subgraph filter bank $\calS(\Gsub)$ plays a pivotal role in achieving accurate approximations, as it determines the expressiveness and flexibility of the learned operator. The design and construction of such filter banks are discussed in detail in \cref{sec:subgraph-filter-algebra}.

\subsection{Empirical Optimization}\label{subsec:empirical}

In practice, when the distribution of $\bx$ is unknown, the stakeholder can only learn the subgraph filter $\Fsub$ from a finite set of training data. We assume that the stakeholder has access to a training dataset $\set{(\bx^i, \by^i)}_{i=1}^n$ consisting of $n$ \gls{iid} samples drawn from the joint distribution of $(\bx,\by)$. The empirically optimal subgraph filter is given by
\begin{align}\label{eq:loss}
\hFsub = \argmin_{\Fsub\in\spn\calS(\Gsub)} \frac{1}{n}\sum_{i=1}^{n} \calL\parens*{\Psub\by^i, \Fsub\Psub\bx^i},
\end{align}
Using \cref{eq:subgraph-filter}, the \gls{SFL} problem \cref{eq:loss} reduces to a finite-dimensional optimization problem over $\set{\theta_j}_{j=1}^m$.

To mitigate potential ill-posedness caused by limited training data, regularization can be employed. A common approach is ridge regression, which introduces a penalty term to control the complexity of the learned filter. The regularized optimization problem is formulated as follows:
\begin{align}\label{eq:loss-reg}
\hFsub=\argmin_{\Fsub\in\spn\calS(\Gsub)} \frac{1}{n}\sum_{i=1}^{n} \calL\parens*{\Psub\by^i, \Fsub\Psub\bx^i}+\eta\norm{\Fsub}_F^2,
\end{align}
where $\eta > 0$ is the regularization parameter, and $\norm{}_F$ denotes the Frobenius norm. This regularization term reduces overfitting and improves numerical stability, particularly when the number of training samples is limited.

\subsection{An Example: Least Squares SFL}\label{subsec:closed-form}

To illustrate the \gls{SFL} framework, we consider least-squares \gls{SFL} with $\calL=L^2$ in \cref{eq:loss} as a representative example. This setting highlights both the empirical optimization procedure and the derivation of analytical solutions.
By concatenating the training samples $\set{(\bx^i, \by^i)}_{i=1}^n$, we obtain
\begin{align*}
\bX_{\Vsub}&=\bmat{\bx^1_{\Vsub}\enspace \bx^2_{\Vsub}\enspace\cdots\enspace \bx^n_{\Vsub}} \in\bbR^{N_0\times n}\\
\bX_{\Vsubc}&=\bmat{\bx^1_{\Vsubc} \enspace \bx^2_{\Vsubc}\enspace\cdots\enspace \bx^n_{\Vsubc}}\in\bbR^{(N-N_0)\times n},
\end{align*}
where $\bX_{\Vsub}$ and $\bX_{\Vsubc}$ collect the samples restricted to $\Vsub$ and its complement, respectively. Similarly, $\bY_{\Vsub}$ and $\bY_{\Vsubc}$ denote the corresponding output-signal matrices.

We assume that the filter space is unconstrained, namely $\spn\calS(\Gsub)=\bbR^{N_0\times N_0}$.
We further assume that the subgraph second-moment matrix
$\Msub\coloneq\E[\bx_{\Vsub}\bx_{\Vsub}\T]$ and the sample Gram matrix $\bX_{\Vsub}\bX_{\Vsub}\T$ are nonsingular, and that $\Psub f(\bx)$ is square-integrable.

The empirical optimization problem in \cref{eq:loss} reduces to
\begin{align*}
\min_{\Fsub} \norm*{\bY_{\Vsub}-\Fsub\bX_{\Vsub}}_F^2.
\end{align*}
This problem admits the closed-form solution
\begin{align}\label{eq:lmmse}
\hFsub=\bY_{\Vsub}\bX_{\Vsub}\T(\bX_{\Vsub}\bX_{\Vsub}\T)^{-1},
\end{align}
which we refer to as the \emph{numerical LMMSE subgraph filter}.
By the strong law of large numbers, $\hFsub \to \Fsub^*$ almost surely (a.s.) as $n\to\infty$, where
\begin{align}
\Fsub^*
&=\E[\by_{\Vsub}\bx_{\Vsub}\T](\E[\bx_{\Vsub}\bx_{\Vsub}\T])^{-1} \nn
&=\E[\Psub f(\bx)\bx_{\Vsub}\T]\Msub^{-1}. \label{eq:oracle-mse}
\end{align}
The operator $\Fsub^*$ is referred to as the \emph{oracle LMMSE subgraph filter}.

We next consider the setting in which $\by=\bF\bx$, where $\bF=\bI-\zeta\bL$ represents one-step graph diffusion and $\bx\in\calN(\bzero,(\bL+\mu\bI)^{-1})$ is a centered \gls{GMRF}. In this case, the oracle LMMSE filter takes the form
\begin{align*}
\Fsub^*=\bI_{\Vsub}-\zeta\bL_{\Vsub\Vsub}+\zeta\bA_{\Vsub\Vsubc}(\bL_{\Vsubc\Vsubc}+\mu\bI_{\Vsubc})^{-1}\bA_{\Vsubc\Vsub}.
\end{align*}
This expression comprises the subgraph diffusion term $\bI_{\Vsub}-\zeta\bL_{\Vsub\Vsub}$ and a boundary correction term $\zeta\bA_{\Vsub\Vsubc}(\bL_{\Vsubc\Vsubc}+\mu\bI_{\Vsubc})^{-1}\bA_{\Vsubc\Vsub}$ induced by signals outside $\Vsub$.
Let $\bD_\mu\coloneq\bD_{\Vsubc\Vsubc}+\mu\bI_{\Vsubc}$. The boundary correction term admits the Neumann series expansion
\begin{align*}
\sum_{k=0}^\infty\bA_{\Vsub\Vsubc}(\bD_\mu^{-1}\bA_{\Vsubc\Vsubc})^k\bD_\mu^{-1}\bA_{\Vsubc\Vsub}.
\end{align*}
The $k$-th term aggregates walks that leave $\Vsub$, take $k$ steps within $\Vsubc$, and then return to $\Vsub$; accordingly, these walks have length $k+2$. Moreover, as shown in \cref{sec:diffusion-details} of the supplementary material, the performance loss satisfies
\begin{align*}
\calR_{L^2}(\Fsub^*)\leq\frac{\zeta^2}{\mu+\lambdamin{\bL_{\Vsubc\Vsubc}}}|\cut(\Vsub,\Vsubc)|,
\end{align*}
where $|\cut(\Vsub,\Vsubc)|=|\set{(u,v)\in E(G)\given u\in\Vsub,v\in\Vsubc}|$ denotes the size of the graph cut set. Thus, the performance loss is jointly determined by the diffusion strength, the \gls{GMRF} regularization, and the connectivity between $\Gsub$ and its complement.

The preceding results show that both the oracle subgraph filter and its achievable error depend jointly on the input statistics and the graph structure.
Although the numerical LMMSE filter is available in closed form, its computation requires inverting the sample Gram matrix $\bX_{\Vsub}\bX_{\Vsub}\T$, which may be singular or ill-conditioned when the available data are limited. This observation motivates learning within a structured finite-dimensional filter class of controllable size. We therefore develop a systematic construction of the filter bank $\calS(\Gsub)$ through subgraph filter algebra and analyze its structural properties.

\section{Subgraph Filter Algebras}\label{sec:subgraph-filter-algebra}

In this section, we construct structured filter banks $\calS(\Gsub)$ for \gls{SFL}. We preface by introducing necessary notations and definitions, followed by the formal definition of subgraph shift operators, which serve as the foundational elements for constructing the filter bank. We then discuss the expressiveness and formal dimension of the resulting subgraph filter algebra. 

\subsection{Preliminaries}\label{subsec:algebra-prelim}

The shortest-path distance between two nodes $u$ and $v$ in $G$ is denoted by $d_G(u,v)$, and the diameter of $G$ is defined as $\diam G = \max\set{d_G(u,v)\given u, v\in G}$. Two nodes $u$ and $v$ are said to be \emph{connected} in $G$ if a path exists between them, and \emph{disconnected} otherwise (we let $d_G(u,v)=\infty$ in this case).

For any matrix $\bB$, we define $\bbR[\bB]=\set{\bM \given \bM=\sum_{l=0}^r a_l\bB^l,\ a_l\in\bbR,\ r\in\bbN}$ and $\bbR_{\leq r}[\bB]=\set{\bM \given \bM=\sum_{l=0}^{r_0} a_l\bB^l,\ a_l\in\bbR,\ r_0\leq r}$.
For matrices $\bB_1,\dots,\bB_k$ of the same dimensions, we define
\begin{align}
\bbR[\bB_1,\dots,\bB_k] &\coloneq\spn\bigcup_{i=1}^k \bbR[\bB_i], \\
\bbR_{\leq r}[\bB_1,\dots,\bB_k] &\coloneq\spn\bigcup_{i=1}^k \bbR_{\leq r}[\bB_i].
\end{align}
Elements in $\bbR[\bB_1,\dots,\bB_k]$ consist of linear combinations of polynomials in individual matrices, without cross-products between distinct matrices.

We define the matrix \emph{free algebra} generated by a set of matrices $\set{\bB_1,\bB_2,\dots,\bB_k}$ of the same dimensions as 
\begin{align}
\bbR_{\leq r}&\langle\bB_1,\bB_2,\dots,\bB_k\rangle \notag\\
&=\spn \set*{\prod_{i=0}^{r}\bM_{l_i}\given \bM_{l_i}\in\set{\bB_1,\bB_2,\ldots,\bB_k} }.\label{def:FA}
\end{align}

\subsection{Distance-based Subgraph Shift Operators}\label{subsec:DSSO}

For $u,v\in\Vsub$, the induced-subgraph distance satisfies $d_{\Gsub}(u,v)\geq d_G(u,v)$, and nodes connected in $G$ may become disconnected in $\Gsub$.  Consequently, relying solely on the adjacency matrix or Laplacian of $\Gsub$ may necessitate higher-order polynomials to propagate information between nodes, or may fail entirely if the nodes are disconnected.

To mitigate these limitations and more effectively approximate the message-passing behavior of the original graph, we propose the use of auxiliary subgraphs derived from both $G$ and $\Gsub$.

\begin{Definition}[DSSO]\label{def:DSSO}
Let $G$ be a graph and $\Gsub$ its induced subgraph. A matrix $\Ssub$ is termed a $k$-hop \gls{DSSO} on $\Gsub$ with respect to $G$ if its $(i,j)$-th entry $(\Ssub)_{i,j} = 0$ for all $v_i, v_j \in \Vsub$ such that $d_G(v_i, v_j) > k$.
\end{Definition}

Note that in \cref{def:DSSO}, the distance metric is based on the ambient graph $G$. For instance, the adjacency matrix and Laplacian of the \emph{induced subgraph} $\Gsub$ are examples of 1-hop \glspl{DSSO}. Additionally, diagonal matrices of size $|\Vsub| \times |\Vsub|$ can be regarded as 0-hop \glspl{DSSO}. The following properties are immediate.

\begin{Corollary}\label{cor:DSSO-traits}
Suppose $\bS_1$ and $\bS_2$ are $k$-hop \glspl{DSSO} of $\Gsub$ over $G$, then $\theta_1\bS_1+\theta_2\bS_2$ is a $k$-hop \gls{DSSO} (of $\Gsub$ over $G$) for any $\theta_1, \theta_2\in\bbR$. 
\end{Corollary}

\begin{Corollary}\label{cor:DSSO-nv}
Suppose $\bS_1$ and $\bS_2$ are $k_1$-hop and $k_2$-hop \glspl{DSSO} of $\Gsub$ over $G$, respectively. Then $\bS_1\bS_2$ is a $(k_1+k_2)$-hop \gls{DSSO}. 
\end{Corollary}

In summary, the distance-dependent behavior of subgraph filters can be systematically controlled through linear combinations and products of \glspl{DSSO}.

We next introduce the notion of distance-$k$ subgraphs, which provide a natural and illustrative example of $k$-hop \glspl{DSSO}.

\begin{Definition}\label{def:dist_sub}
For $1\leq k\leq\diam G$, let the distance-$k$ subgraph of $\Gsub$ be a graph such that its vertex set is $\Vsub$; and its edge set is $\set*{(u,v)\in \Vsub\times \Vsub\given d_G(u,v)=k}$, i.e., the edges are those pairs of nodes in $\Vsub$ that are $k$ hops apart in the ambient graph $G$.  
We denote this subgraph as $\Gsub^{[k]}$.
\end{Definition}

\begin{figure}[!tb]
\centering
\includegraphics[width=.65\columnwidth]{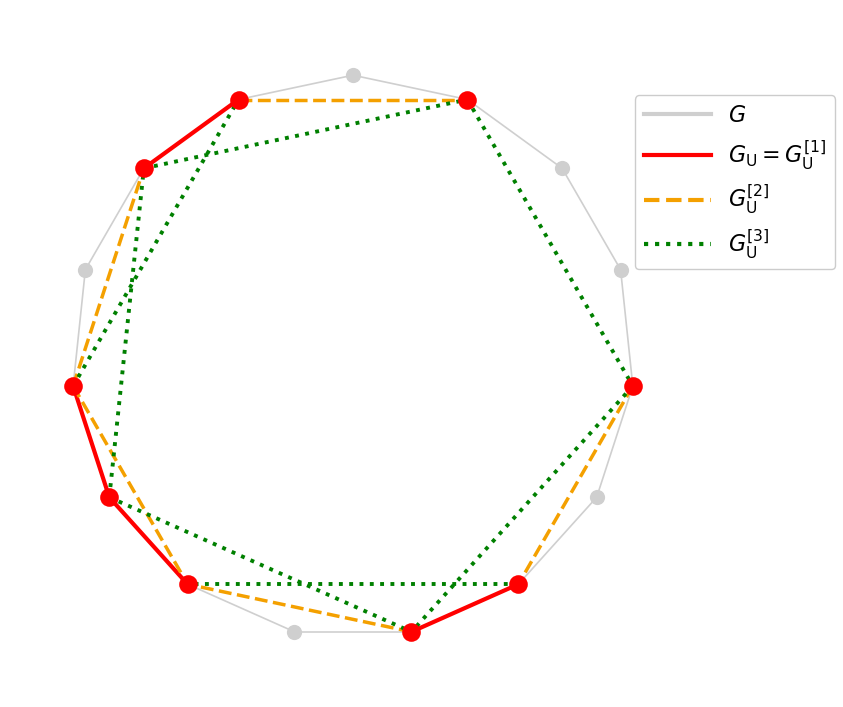}
\caption{Example for distance-based subgraphs on a cyclic graph $G$ with $15$ nodes.}
\label{fig:dist_sub}
\vspace{-3mm}
\end{figure}

\Cref{fig:dist_sub} illustrates the construction of $\Gsub^{[k]}$ on a cycle graph for $k\in\{1,2,3\}$. In this example, $\Gsub$ comprises five connected components, including two isolated vertices. When filter design is restricted to the induced subgraph \gls{GSO} of $\Gsub$, information transfer across disconnected components is prevented. In contrast, \glspl{DSSO} employing distance-$k$ subgraphs enable cross-component propagation, thereby enhancing the representational capacity of the resulting filters.

\subsection{DSSO Algebra}\label{subsec:ssoa}

As discussed above, $\Gsub$ may be disconnected, thereby limiting information exchange across components. A natural remedy is to combine operators associated with different distance-$k$ subgraphs. This leads to a non-commutative algebraic framework for constructing filters from multiple \glspl{DSSO}.

If filtering on $\Gsub\subseteq G$ is restricted to a single subgraph \gls{GSO} $\Lsub$, then the filter space is the commutative algebra $\bbR[\Lsub]$. Its operators share a common eigenspace, and its dimension is at most $|\Vsub|$. These structural constraints may limit expressiveness.
By contrast, we select \glspl{DSSO} as \emph{primitives} that are not jointly diagonalizable. The resulting non-commutative algebra is considerably richer, yielding a broader and more expressive class of subgraph filters.

\begin{Definition}\label{def:sffa}
Let $\set{\Ssub[1],\Ssub[2],\dots,\Ssub[k]}$ be a set of $k$-\glspl{DSSO} as primitives.
The $(k,r)$-\gls{SFFA} of $\Gsub$ is given by (recall the notation in \cref{def:FA}):
\begin{align}\label{eq:sffa}
\scA_{k, r} = \bbR_{\leq r}\langle \Ssub[1], \Ssub[2],\ldots,\Ssub[k]\rangle.
\end{align}
\end{Definition}

A natural strategy of selecting primitives is to use \glspl{DSSO} induced by $\set{\Gsub^{[i]}}_{i=1}^{k}$. For instance, their Laplacians naturally define the primitive set $\set{\Lsub^{[i]}}_{i=1}^{k}$. This choice allows the resulting filter algebra to combine multiple distance-dependent propagation patterns while keeping the primitive construction simple and fully supported on $\Vsub$.

\subsection{Expressiveness and Dimensionality Control}\label{subsec:expressiveness}

Next, we address the question of whether \gls{SFFA} can span the entire space of linear operators on $\mathbb{R}^{|\Vsub|\times|\Vsub|}$, which is pivotal for understanding the theoretical capabilities and practical applications of \gls{SFFA} in \gls{SFL}.

\begin{Theorem}[Expressiveness of \gls{SFFA}]\label{thm:full-exp}
Suppose $G$ is a connected graph, and $\Gsub$ is the induced subgraph on $\Vsub$. Define $\bE_{ij}\in\set{0,1}^{|\Vsub|\times|\Vsub|}$ as the elementary matrix with a $1$ in the $(i,j)$-th position and $0$ elsewhere. Consider the set of primitives
\begin{align*}
\calU=\set{\bE_{ii}}_{i=1}^{|\Vsub|}\cup\set{\bA_k}_{k=1}^{\diam G},
\end{align*}
where $\bA_k$ is the adjacency matrix of the distance-$k$ subgraph $\Gsub^{[k]}$. Then, the $(|\calU|,3)$-\gls{SFFA} satisfies
\begin{align*}
\dim \scA_{|\calU|, 3} = \dim \bbR_{\leq 3}\langle\calU\rangle = |\Vsub|^2.
\end{align*}
\end{Theorem}
\begin{proof}
First, it is immediate that $\dim \bbR_{\leq 3}\langle\calU\rangle\leq|\Vsub|^2$. Since $\bbR_{\leq 3}\langle\calU\rangle$ already contains all diagonal elementary matrices $\bE_{ii}$, it remains to show that it can represent every off-diagonal elementary matrix $\bE_{ij}$ for $i\neq j$, with $i,j\in\Vsub$. Because $G$ is connected, $d_G(i,j)\leq \diam G<\infty$. Let $d=d_G(i,j)$. Then $\bA_d\in\calU$, and its $(i,j)$ entry is equal to $1$. Moreover, since $\bE_{ii}$ and $\bE_{jj}$ are primitives, we have $\bE_{ij}=\bE_{ii}\bA_d\bE_{jj}$, which is represented by at most three primitives from $\calU$. Therefore,
\begin{align*}
\bbR_{\leq 3}\langle\calU\rangle \supset \spn \set{\bE_{ij}}_{i,j=1}^{|\Vsub|}.
\end{align*}
Hence,
\begin{align*}
|\Vsub|^2\geq \dim\bbR_{\leq 3}\langle\calU\rangle \geq \dim\spn \set{\bE_{ij}}_{i,j=1}^{|\Vsub|}=|\Vsub|^2,
\end{align*}
which concludes the proof.
\end{proof}

\begin{Remark}\label{rmk:lap-full-exp}
\cref{thm:full-exp} also holds if we replace the adjacency matrix with Laplacian, normalized adjacency, or normalized Laplacian matrix of $\Gsub^{[k]}$.
\end{Remark}

\cref{thm:full-exp} establishes that, by incorporating diagonal ($0$-hop) \glspl{DSSO} and all distance-based subgraph primitives, the resulting \gls{SFFA} can span the entire matrix space. 

The next step is to determine the dimension of \gls{SFFA} filter banks. Accurately determining their dimension is challenging due to potential linear dependencies among monomials (e.g., while $\bS_1$ and $\bS_2$ may not commute, monomials such as $\bS_1\bS_2\bS_1$ and $\bS_2\bS_1^3$ could still be linearly dependent), so we instead focus on the \emph{formal dimension} of the filter algebra, which provides a practical measure directly related to the number of trainable parameters when optimizing filters within the space using optimization techniques. 

\begin{Definition}[Formal Dimension]\label{def:form-dim}
The formal dimension of a $(k,r)$-\gls{SFFA} $\scA_{k,r}$ is defined as the total number of distinct monomials of degree at most $r$ generated by the primitives, including the identity.

\end{Definition}

\begin{Remark}\label{rmk:formal-dimension}
The formal dimension is an upper bound on the true dimension of the evaluated filter space, since each monomial in \gls{SFFA}, and each reversal-equivalence class in sym-\gls{SFFA}, corresponds to one evaluated matrix generator. The bound may be strict when these generators are linearly dependent. In sym-\gls{SFFA}, accidental symmetry of monomials only further reduces the actual dimension.
\end{Remark}

\begin{Theorem}\label{prop:expressive}
For $k>1$, the formal dimension of a $(k,r)$-\gls{SFFA} is  $\sum_{l=0}^r k^l = \frac{k^{r+1}-1}{k-1}.$
\end{Theorem}

\begin{proof}

For a $(k,r)$-\gls{SFFA}, a monomial of degree $l$ is formed by multiplying $l$ primitives. Since there are $k$ available choices for each factor, there are exactly $k^l$ distinct monomials of degree $l$. Summing these quantities over all degrees from $l=0$ to $r$ yields the stated formal dimension. 
\end{proof}

For example, the sum of polynomial spaces $\sum_{i=1}^k\bbR_{\leq r_i}[\bS_i]$ has a formal dimension of $1+\sum_{i=1}^kr_i$, as the identity matrix is shared. For $\bbR_{\leq 3}\langle\calU\rangle$ in \cref{thm:full-exp}, the formal dimension is $\sum_{l=0}^3(|\Vsub|+\diam G)^l$, which vastly exceeds the true dimension $|\Vsub|^2$, indicating significant linear dependence among monomials. 

In practice, to make sure the choice of $k$ and $r$ reflects the underlying graph structure, it is advised to incorporate more distance-based primitives (larger $k$) while keeping $r$ small (cf. \cref{subsubsec:param-sweep}). Conversely, for denser graphs (e.g., Erd\H{o}s–Rényi graphs with $p > 0.3$), a smaller $k$ together with a higher polynomial order $r$ is often sufficient. 

In finite-sample regimes, estimating a larger number of parameters exacerbates estimation error and introduces numerical instability. Thus, the choice of filter algebra dictates a fundamental bias-variance trade-off between approximation capacity (determined by algebraic expressiveness) and statistical efficiency (determined by the formal dimension). We formalize this trade-off in the next section.

\section{Performance Risk Bounds}\label{sec:error-bound}

In the \gls{SFL} framework, estimating an optimal filter from a prescribed filter bank $\calS(\Gsub)$ using finite samples inherently induces a bias-variance trade-off. A restricted filter class yields approximation error due to insufficient representational capacity, whereas an overly expressive class amplifies estimation error in the finite-sample regime. The following results provide formal excess-risk bounds that characterize how the approximation of the oracle operator depends on the filter-bank dimension, sample size, and data distribution. We make the following assumption throughout this section.

\begin{Assumption}\label{asm:error-bound}
The input signal $\bx$ is zero-mean, sub-Gaussian, and has a positive-definite second-moment matrix $\bM$. In addition, $\bx$ is graph wide-sense stationary\cite{Gir:C15,PerVan:J17}, i.e., $\bM$ and the graph Laplacian $\bL$ are simultaneously diagonalizable. The noise vector $\bw$ is sub-Gaussian and uncorrelated with $f(\bx)$. The mapping $f$ is globally Lipschitz: $\norm{f(\bu)}_2 \le L_f\norm{\bu}_2$, for all $\bu\in\bbR^N$, and for some constant $L_f>0$.
\end{Assumption}

Let $\calV_k$ denote the linear space of $k$-hop \glspl{DSSO}, and let $m \coloneq\dim(\calV_k)$.
Let $\set{\bQ_j}_{j=1}^m$ be a Frobenius-orthonormal basis of $\calV_k$, i.e., $\angles{\bQ_j,\bQ_{j'}}_F=\delta_{jj'}$. Every $\bF\in\calV_k$ then admits the expansion $\bF=\sum_{j=1}^m\theta_j\bQ_j$.

Let $\Fsub^*$ denote the oracle LMMSE subgraph filter. Define
\begin{gather*}
\bar{\bF}_m\coloneq \argmin_{\bF\in\calV_k} \calR_{L^2}(\bF),\\
\bar{\bF}_{m,\eta} \coloneq \argmin_{\bF\in\calV_k} \braces*{\calR_{L^2},(\bF)+\eta\norm{\bF}_F^2},
\end{gather*}
and the empirical ridge estimator
\begin{align}\label{eq:empirical-ridge-estimator}
\hat{\bF}_{m,\eta}\tc{n}
\coloneq\argmin_{\bF\in\calV_k}
\braces*{ \frac{1}{n}\sum_{i=1}^n \norm{\by_{\Vsub}^i-\bF\bx_{\Vsub}^i}_2^2 + \eta\norm{\bF}_F^2}.
\end{align}
To quantify the component of a filter that is activated by the input distribution, we introduce the $\Msub$-induced norm and inner product:
\begin{align}\label{eq:input-norm}
\norm{\bS}_{\Msub}^2\coloneq\E\norm{\bS\bx_{\Vsub}}_2^2=\tr(\bS\Msub\bS\T)
\end{align}
and
\begin{align}\label{eq:input-prod}
\angles{\bS_1,\bS_2}_{\Msub}\coloneq\E[(\bS_2\bx_\Vsub)\T(\bS_1\bx_\Vsub)]=\tr(\bS_1\Msub\bS_2\T).
\end{align}
For notational convenience, let $\kappa_{\min}\coloneq\lambdamin{\bM}$ and $\kappa_{\max}\coloneq\lambdamax{\bM}$ be the minimum and maximum eigenvalues of $\bM$, respectively. We denote by $\norm{\cdot}_{\psi_2}$ the sub-Gaussian norm \cite[Definition 2.5.6]{Ver18} and by $\norm{\cdot}_{\psi_1}$ the sub-exponential norm \cite[Definition 2.7.5]{Ver18}. 

For the coordinate representation over $\calV_k$, we also define the Gram matrix
$\Gamma\in\bbR^{m\times m}$ and the coordinate response vector $\bb\in\bbR^m$ by
\begin{align}
\Gamma_{j\ell}
&\coloneq
\angles{\bQ_j,\bQ_\ell}_{\Msub} \label{eq:coord-gram}\\
b_j
&\coloneq
\E[\by_{\Vsub}\T\bQ_j\bx_{\Vsub}],
\qquad
\bb\coloneq(b_1,\ldots,b_m)\T .
\label{eq:coord-response}
\end{align}

We now establish bounds on the least squares performance loss of the empirical ridge estimator in \cref{eq:empirical-ridge-estimator}. Proof sketches for \cref{thm:hp-ridge-decomp} and complete proofs of \cref{thm:decomposition-bounds}\ref{thm:hp-ridge-oracle} are given in \cref{sec:selected-proofs}; the remaining proofs are deferred to \cref{sec:hp-ridge} of the supplementary material. We start with a few preliminary results.

\begin{Lemma}\label{lma:gram-eig}
Under \cref{asm:error-bound}, the Gram matrix $\Gamma$ satisfies
\begin{equation}\label{eq:gram-eig}
\begin{aligned}
\kappa_{\min}\leq\lambdamin{\Msub}\leq\lambdamin{\Gamma};\\
\lambdamax{\Gamma}\leq\lambdamax{\Msub}\leq\kappa_{\max}.
\end{aligned}
\end{equation}
\end{Lemma}
\Cref{lma:gram-eig} establishes the numerical stability of the coordinate system induced by $\calV_k$: the eigenvalues of the Gram matrix $\Gamma$ are sandwiched between those of $\Msub$, and therefore lie within $[\kappa_{\min}, \kappa_{\max}]$. This conditioning guarantee is invoked repeatedly throughout the analysis to control the norms of inverse Gram matrices and to translate errors in the coordinate coefficients into prediction errors.

\begin{Lemma}\label{lma:b-coordinate-norm}
Under the model $\by=f(\bx)+\bw$ with $\E[\bw\bx\T]=\bzero$, the coordinate response vector $\bb$ satisfies
\begin{align}\label{eq:b-coordinate-norm}
\norm{\bb}_2^2
=
\sum_{\substack{u,v\in\Vsub\\ d_G(u,v)\leq k}} \parens*{(\bJ\bM)_{uv}}^2,
\end{align}
where $\bJ\coloneq\E[f(\bx)\bx\T]\bM^{-1}$.
\end{Lemma}
\Cref{lma:b-coordinate-norm} connects the ridge analysis to graph distance. Specifically, it identifies $\bb$ as the $k$-hop component of the input--output cross-moment; consequently, only correlations between node pairs with $d_G(u,v)\leq k$ contribute to the coordinate response. This is precisely where the graph-local constraint of \gls{SFL} enters the subsequent bounds.

\begin{Theorem}[Ridge excess risk decomposition]\label{thm:hp-ridge-decomp}
Suppose \cref{asm:error-bound} holds. Then, for any $\eta>0$,
\begin{multline}
\E\norm{\Psub f(\bx)-\hat{\bF}_{m,\eta}\tc{n}\bx_{\Vsub}}_2^2 \\
\leq
\calR_{L^2}(\Fsub^*) 
+ \norm{\Fsub^*-\bar{\bF}_m}_{\Msub}^2  
+ 2\norm{\bar{\bF}_m-\bar{\bF}_{m,\eta}}_{\Msub}^2 \\
+ 2\norm{\bar{\bF}_{m,\eta}-\hat{\bF}_{m,\eta}\tc{n}}_{\Msub}^2. \label{eq:hp-ridge-full}
\end{multline}
\end{Theorem}

\Cref{thm:hp-ridge-decomp} decomposes the prediction risk into four components:
\begin{itemize}
\item \emph{Performance error} $\calR_{L^2}(\Fsub^*)$: the irreducible discrepancy between $\Psub f(\bx)$ and the best linear predictor based on $\bx_{\Vsub}$; this term is independent of both the filter-bank class and the sample size.

\item \emph{Approximation error} $\norm{\Fsub^*-\bar{\bF}_m}_{\Msub}^2$: the error induced by constraining the oracle LMMSE filter to the $k$-hop subspace $\calV_k$, i.e., the component of $\Fsub^*$ that cannot be represented in $\calV_k$. This term decreases as $k$ increases.

\item \emph{Regularization bias} $2\norm{\bar{\bF}_m-\bar{\bF}_{m,\eta}}_{\Msub}^2$: the bias introduced by ridge regularization; it vanishes in the limit $\eta\to0^+$.

\item \emph{Empirical estimation error} $2\norm{\bar{\bF}_{m,\eta}-\hat{\bF}_{m,\eta}\tc{n}}_{\Msub}^2$: finite-sample deviation of the empirical estimator from its population counterpart. This term typically decreases with $n$ and may increase with the filter dimension $m$; hence, enlarging $\calV_k$ can reduce approximation error while potentially increasing estimation error.
\end{itemize}

In what follows, we derive bounds for each error component as functions of the filter space $\calV_k$, through the following quantities:
\begin{itemize}
    \item $\sigma_\Vsub^2 \coloneq \E\norm{\Psub\parens*{f(\bx)-\bJ\bx}}_2^2$ measures the intrinsic residual of $f$ after its oracle linear approximation $\bJ\bx$, and thus captures the performance error that is independent of the choice of $\calV_k$.

    \item $N_{>k}(\Vsub;G)\coloneq\abs*{\set{(u,v)\in\Vsub\times\Vsub\given d_G(u,v)>k}}$, defined as the number of \emph{ordered} node pairs in $\Vsub$ that are more than $k$ hops apart in $G$, quantifies the number of within-subgraph filter entries excluded by the $k$-hop support constraint defining $\calV_k$.

    \item $s_\partial \coloneq\max_{u\in\Vsub,v\in\Vsubc}\abs*{\bJ_{uv}}$, with $s_\partial=0$ when $\Vsub=\varnothing$ or $\Vsubc=\varnothing$, quantifies the strength of boundary dependence from $\Vsubc$ to $\Vsub$, and captures the performance error induced by boundary propagation between $\Vsub$ and $\Vsubc$.

    \item $s_c\coloneq\max_{u,v\in\Vsub,d_G(u,v)>k}\abs*{\bJ_{uv}}$, with $s_c=0$ when $N_{>k}(\Vsub;G)=0$, quantifies the strength of within-subgraph oracle interactions that lie outside the representable $k$-hop support of $\calV_k$.
   
\end{itemize}

\begin{Theorem}\label{thm:decomposition-bounds}
Suppose \cref{asm:error-bound} holds. Then, we have the following bounds:
\begin{enumerate}[(i)]
\item\label[bound]{thm:hp-ridge-oracle} 
The intrinsic performance error
\begin{align}
\calR_{L^2}(\Fsub^*)\leq\sigma_{\Vsub}^2+\kappa_{\max}|\Vsub||\Vsubc|s_{\partial}^2.
\end{align}

\item\label[bound]{thm:hp-ridge-approx} 
The approximation error
\begin{align}
\ml{\norm{\Fsub^*-\bar{\bF}_m}_{\Msub}^2\\
\leq\kappa_{\max}N_{>k}(\Vsub;G)\left(s_c+s_{\partial}\norm{\bM_{\Vsubc\Vsub}\Msub^{-1}}_1\right)^2.}
\end{align}

\item\label[bound]{thm:hp-ridge-bias} 
The ridge regularization error
\begin{align}\label{eq:hp-ridge-bias}
\ml{\norm{\bar{\bF}_m-\bar{\bF}_{m,\eta}}_{\Msub}^2 \\
\leq \frac{\eta^2}{\kappa_{\min}(\kappa_{\min}+\eta)^2}\sum_{\substack{u,v\in\Vsub\\d_G(u,v)\leq k}}\parens*{\bJ\bM}_{uv}^2.}
\end{align}

\item\label[bound]{thm:hp-ridge-estimation}
For any $\beta\in(0,1)$, if $n$ is sufficiently large, then with probability at least $1-\beta$, the estimation error is bounded as
\begin{align*}
\ml{\norm{\bar{\bF}_{m,\eta}-\hat{\bF}_{m,\eta}\tc{n}}_{\Msub}\\
\leq\frac{\sqrt{\kappa_{\max}}}{\kappa_{\min}/2+\eta}
\bigg[\calO \parens*{\sqrt{\frac{m\log(m/\beta)}{n}}}\\
+\calO \parens*{m\sqrt{\frac{\log(m/\beta)}{n}}}\Theta_\eta\bigg],}
\end{align*}
where $\calO(\cdot)$ is the big-O notation, and
\begin{align*}
\Theta_\eta\coloneq
\frac{1}{\kappa_{\min}+\eta}\parens*{\sum_{\substack{u,v\in\Vsub \\ d_G(u,v)\leq k}}\parens*{\bJ\bM}_{uv}^2}^{\ofrac{2}}.
\end{align*}

\end{enumerate}
\end{Theorem}

As shown in \cref{thm:decomposition-bounds}, increasing $k$ can reduce the approximation error but enlarges the filter-space $\calV_k$ with dimension $m$, potentially increasing the estimation error in finite samples. Likewise, decreasing $\eta$ reduces the regularization bias but may increase the estimation error. The overall risk therefore balances the approximation error, regularization bias, and estimation error through $k$, $m$, $n$, and $\eta$.

The first two bounds, \cref{thm:hp-ridge-oracle,thm:hp-ridge-approx}, characterize the performance error and approximation error specific to \gls{SFL}. The former quantifies the performance error from predicting with only $\bx_{\Vsub}$, while the latter quantifies the approximation error introduced by enforcing $k$-hop support. Together, they relate these errors to subgraph restriction, graph distance, and the input-output distribution.

The remaining bounds, \cref{thm:hp-ridge-bias,thm:hp-ridge-estimation}, characterize the regularization bias and estimation error within a fixed filter space. They quantify, respectively, the effect of ridge regularization and the finite-sample deviation governed by $n$, $m$, $\eta$, and data conditioning.

\section{Numerical Experiments}\label{sec:experiments}

In this section, we apply \gls{SFL} in filtering, reconstruction and prediction tasks.\footnote{The code is available at \url{https://github.com/temperierte/DSSO}.} We compare the proposed \gls{SFFA}-based \gls{SFL} filters with data-agnostic filters, structure-agnostic filters, and polynomial-based filters.

\subsection{Subgraph Filter Banks and Baselines}\label{subsec:exp-filter-banks}

\begin{table}[!htb]
\centering
\caption{Primitives used in experiments.}
\begin{tabular}{p{0.25\columnwidth} p{0.65\columnwidth}}
\hline
Primitives on $\Gsub$ & Description \\
\hline
$\tbL_{\mathrm{ind}}$ 
& the normalized induced subgraph Laplacian on $\Gsub$ \\

$\tbL_{\mathrm{Kron}}$ 
& the normalized Kron-reduced Laplacian on $\Vsub$ derived from $G$ \\

$\tbL_{\mathrm{rand},i}$, $i=1,2,3,\ldots$ 
& normalized Laplacians from independently sampled $p_0=0.9$ induced sub-subgraphs of $\Vsub$, zero-padded to $\Vsub$ \\

$\tbL_{\Vsub}^{\cup_k}$ 
& the normalized Laplacian of the union of the distance-$1,2,3,\ldots,k$ subgraphs \\

$\tbL_{\Vsub}^{[k]}$ 
& normalized distance-$k$ subgraph Laplacian \\
\hline
\end{tabular}
\vspace{-3mm}
\label{tab:sfl-filtering-primitives}
\end{table}

The primitives used to construct filter banks are listed in \cref{tab:sfl-filtering-primitives}. Recall that $\tbL$ denotes the normalized version of $\bL$. Given the number of primitives $k$ and the maximum order $r$, we describe baseline \gls{SFL} methods and our proposed \gls{SFFA} below. 

\begin{itemize}
\item \emph{Numerical LMMSE} in \cref{eq:lmmse} corresponds to optimization over the maximal subgraph filter space $\mathbb{R}^{N_0\times N_0}$ with $N_0^2$ trainable parameters, and is included only as a reference.
\item \emph{Single-\gls{GSO} polynomials} denoted as $\bbR_{\leq r}[\tbL_{\mathrm{ind}}],\ 
\bbR_{\leq r}[\tbL_{\mathrm{Kron}}],\ 
\bbR_{\leq r}[\tbL_\Vsub^{\cup_k}]$, which are order-$r$ polynomial filters of different \glspl{GSO}. 
\item \emph{Distance-based subgraph Laplacian polynomials} denoted as $\bbR_{\leq r}[\tbL_\Vsub^{[1]},\tbL_\Vsub^{[2]},\ldots,\tbL_\Vsub^{[k]}]$, cf.\ \cref{subsec:algebra-prelim}.
\item \emph{Random-group Laplacian algebra.} We use $\bbR_{\leq r}\langle\tbL_{\mathrm{rand},1},\tbL_{\mathrm{rand},2},\ldots,\tbL_{\mathrm{rand},k}\rangle$, where $\tbL_{\mathrm{rand},i}$ denotes the normalized Laplacian of the induced subgraph of $\Gsub$ obtained by sampling a $\gamma_0=0.9$ fraction of nodes from $\Vsub$, zero-padded to $\Vsub$, for $i=1,\dots,k$. It is a baseline to build up a matrix algebra structure with randomized subgraph connections.
\item \emph{Distance-based subgraph Laplacian \gls{SFFA}.} Our proposed method $\bbR_{\leq r}\langle\tbL_\Vsub^{[1]},\tbL_\Vsub^{[2]},\ldots,\tbL_\Vsub^{[k]}\rangle$ uses normalized distance-based subgraph Laplacians and the \gls{SFFA} filter bank structure.
\end{itemize}

\subsection{Semi-Synthetic Subgraph Filtering}\label{subsec:exp-sfl}

We use the graph $G$ and signals $\bx$ from the METR-LA dataset \cite{li18}. After removing edge directions, the resulting graph contains $N=207$ nodes. We retain only the time stamps with complete sensor readings and split the data chronologically using a 7:3 ratio for train and test, respectively. 

We evaluate the behavior of different subgraph filter banks in a setting with a nonlinear ground-truth function, non-Gaussian skewed noise, and a nontrivial loss function. To construct the output $\by$, we let
\begin{align*}
f(\bx)
&= 20(\bI-\tbL)^2\brk*{\parens*{\tbL\bx}\odot\tanh\parens*{\frac{\tbL^2\bx}{\sqrt{\frac{1}{N}\norm{\tbL^2\bx}_2^2}}}};\\
\bw
&=20\varepsilon;\\
\calL(\hat\by_\Vsub,\by_\Vsub)
&= \frac{1}{N_0}\sum_{j=1}^{N_0}\ell_{\tau,\delta}(\hat y_j-y_j),
\end{align*}
where
\begin{gather*}
\varepsilon = \frac{(B_i-\rho)}{\sqrt{\rho(1-\rho)}},\quad
B \sim\mathrm{Bernoulli}(\rho),\quad
\rho=0.8;\\
\ml{
\ell_{\tau,\delta}(u)
= \delta \cdot \bigg(
\tau\ln\parens*{1+\exp\parens*{\frac{u}{\delta}}}\\
{}+(1-\tau)\ln\parens*{1+\exp\parens*{-\frac{u}{\delta}}}
\bigg),
\quad \tau=0.8,\quad \delta=0.1.
}
\end{gather*}

For each trial of \gls{SFL} filtering, we randomly sample a subset $\Vsub\subset V$ containing $N_0=\gamma N$ nodes to define the subgraph $\Gsub$, with $\gamma=0.1,0.2,\ldots,0.9$.
We instantiate the subgraph filter bank according to \cref{subsec:exp-filter-banks}. To control the maximum size of the \gls{SFFA}-based filter bank, we set the parameter budget to $\alpha N_0^\beta$ with $\alpha=1$ and $\beta=1.5$ as specified in \cref{eq:budget}. To determine the $(k,r)$ pair for the \gls{SFFA}-based filter for each given $N_0$, we fix $r=3$ (which provides non-trivial filtering without over-relying on distance-based subgraph Laplacians) and maximize $k$. 
We compare the proposed distance-based subgraph Laplacian \gls{SFFA}-based filters against other \gls{SFL} methods using the same loss function for training.

\begin{table*}[!t]
\centering
\caption{Test loss with standard deviation for semi-synthetic subgraph filtering on METR-LA with parameter budget $1\cdot N_0^{1.5}$, with the \red{best}, \blue{second}, and \cyan{third} best performance highlighted.}
\label{tab:exp_sfl_filtering}
\setlength{\tabcolsep}{4pt}
\renewcommand{\arraystretch}{1.08}

\resizebox{\textwidth}{!}{%
\begin{tabular}{lccccccccc}
\toprule
Subgraph ratio
& $\gamma=0.1$
& $\gamma=0.2$
& $\gamma=0.3$
& $\gamma=0.4$
& $\gamma=0.5$
& $\gamma=0.6$
& $\gamma=0.7$
& $\gamma=0.8$
& $\gamma=0.9$ \\
$(k,r)$
& $(4,3)$
& $(6,3)$
& $(7,3)$
& $(8,3)$
& $(9,3)$
& $(10,3)$
& $(11,3)$
& $(12,3)$
& $(13,3)$ \\
\midrule
Numerical LMMSE
& $20.06\pm0.87$
& $\third{18.14\pm0.87}$
& $\third{16.58\pm0.69}$
& $\third{15.37\pm0.52}$
& $\third{14.61\pm0.54}$
& $\second{14.01\pm0.44}$
& $\second{13.66\pm0.33}$
& $\second{13.16\pm0.17}$
& $\second{12.90\pm0.14}$ \\
\midrule
$\bbR_{\leq r}[\tbL_{\mathrm{ind}}]$
& $20.14\pm1.43$
& $20.69\pm1.22$
& $20.58\pm1.01$
& $20.55\pm0.72$
& $20.49\pm0.53$
& $20.34\pm0.45$
& $20.51\pm0.34$
& $20.50\pm0.15$
& $20.42\pm0.13$ \\
$\bbR_{\leq r}[\tbL_{\mathrm{Kron}}]$
& $19.93\pm1.52$
& $20.31\pm1.22$
& $20.38\pm0.95$
& $20.41\pm0.63$
& $20.40\pm0.47$
& $20.28\pm0.44$
& $20.47\pm0.30$
& $20.47\pm0.15$
& $20.41\pm0.13$ \\
$\bbR_{\leq r}[\tbL_\Vsub^{\cup_k}]$
& $19.58\pm1.55$
& $20.02\pm1.15$
& $19.92\pm0.89$
& $19.56\pm0.65$
& $19.47\pm0.48$
& $19.42\pm0.47$
& $19.54\pm0.34$
& $19.48\pm0.16$
& $19.39\pm0.11$ \\
\midrule
$\bbR_{\leq r}[\tbL_\Vsub^{[1]},\tbL_\Vsub^{[2]},\ldots,\tbL_\Vsub^{[k]}]$
& $\third{18.69\pm1.60}$
& $18.31\pm0.79$
& $18.24\pm0.90$
& $17.98\pm0.76$
& $17.35\pm0.56$
& $17.06\pm0.49$
& $17.03\pm0.27$
& $17.05\pm0.19$
& $17.00\pm0.16$ \\
$\bbR_{\leq r}\langle\tbL_{\mathrm{rand},1},\tbL_{\mathrm{rand},2},\ldots,\tbL_{\mathrm{rand},k}\rangle$
& $\second{18.01\pm1.88}$
& $\second{15.07\pm0.86}$
& $\second{14.55\pm0.70}$
& $\second{14.38\pm0.51}$
& $\second{14.23\pm0.33}$
& $\third{14.16\pm0.51}$
& $\third{14.05\pm0.32}$
& $\third{14.01\pm0.47}$
& $\third{14.01\pm0.34}$ \\
\midrule
$\bbR_{\leq r}\langle\tbL_\Vsub^{[1]},\tbL_\Vsub^{[2]},\ldots,\tbL_\Vsub^{[k]}\rangle$
& $\first{14.51\pm0.69}$
& $\first{13.97\pm0.40}$
& $\first{13.58\pm0.35}$
& $\first{13.29\pm0.28}$
& $\first{13.16\pm0.25}$
& $\first{13.01\pm0.32}$
& $\first{12.97\pm0.35}$
& $\first{12.93\pm0.23}$
& $\first{12.81\pm0.14}$ \\
\bottomrule
\end{tabular}%
}
\vspace{-3mm}
\end{table*}

\begin{table*}[!t]
\centering
\caption{Holm-Bonferroni-adjusted $p$-values of significance test over proposed \gls{SFFA} in subgraph filtering task.}
\label{tab:exp_sfl_filtering_pval}
\setlength{\tabcolsep}{2pt}

\resizebox{\textwidth}{!}{%
\begin{tabular}{lccccccccc}
\toprule
Subgraph ratio
& $\gamma=0.1$
& $\gamma=0.2$
& $\gamma=0.3$
& $\gamma=0.4$
& $\gamma=0.5$
& $\gamma=0.6$
& $\gamma=0.7$
& $\gamma=0.8$
& $\gamma=0.9$ \\
$(k,r)$
& $(4,3)$
& $(6,3)$
& $(7,3)$
& $(8,3)$
& $(9,3)$
& $(10,3)$
& $(11,3)$
& $(12,3)$
& $(13,3)$ \\
\midrule
$\bbR_{\leq r}[\tbL_{\mathrm{ind}}]$
& $1.22{\times}10^{-7}$
& $3.51{\times}10^{-8}$
& $3.81{\times}10^{-9}$
& $8.40{\times}10^{-10}$
& $4.92{\times}10^{-11}$
& $7.65{\times}10^{-12}$
& $4.53{\times}10^{-12}$
& $1.52{\times}10^{-14}$
& $1.78{\times}10^{-17}$ \\
$\bbR_{\leq r}[\tbL_{\mathrm{Kron}}]$
& $3.30{\times}10^{-7}$
& $3.29{\times}10^{-8}$
& $3.05{\times}10^{-9}$
& $4.55{\times}10^{-10}$
& $2.32{\times}10^{-11}$
& $4.16{\times}10^{-12}$
& $1.80{\times}10^{-12}$
& $1.18{\times}10^{-14}$
& $1.67{\times}10^{-17}$ \\
$\bbR_{\leq r}[\tbL_\Vsub^{\cup_k}]$
& $3.16{\times}10^{-7}$
& $2.93{\times}10^{-8}$
& $3.11{\times}10^{-9}$
& $1.19{\times}10^{-9}$
& $4.61{\times}10^{-11}$
& $1.10{\times}10^{-11}$
& $6.32{\times}10^{-12}$
& $3.14{\times}10^{-14}$
& $1.38{\times}10^{-17}$ \\
\midrule
$\bbR_{\leq r}[\tbL_\Vsub^{[1]},\tbL_\Vsub^{[2]},\ldots,\tbL_\Vsub^{[k]}]$
& $1.40{\times}10^{-6}$
& $1.94{\times}10^{-8}$
& $8.53{\times}10^{-8}$
& $1.95{\times}10^{-8}$
& $9.13{\times}10^{-9}$
& $2.45{\times}10^{-10}$
& $9.47{\times}10^{-11}$
& $1.30{\times}10^{-12}$
& $1.81{\times}10^{-15}$ \\
$\bbR_{\leq r}\langle\tbL_{\mathrm{rand},1},\tbL_{\mathrm{rand},2},\ldots,\tbL_{\mathrm{rand},k}\rangle$
& $3.29{\times}10^{-5}$
& $2.17{\times}10^{-3}$
& $4.94{\times}10^{-4}$
& $2.72{\times}10^{-5}$
& $1.85{\times}10^{-6}$
& $4.87{\times}10^{-7}$
& $8.20{\times}10^{-6}$
& $7.40{\times}10^{-6}$
& $8.10{\times}10^{-7}$ \\
\bottomrule
\end{tabular}%
}
\vspace{-3mm}
\end{table*}

From \cref{tab:exp_sfl_filtering}, we have the following observations:
\begin{itemize}
\item The numerical LMMSE is competitive yet suboptimal even with sufficient training data. This is because it estimates the oracle LMMSE filter in \cref{eq:lmmse}, which is optimal only under the least squares loss criterion.

\item The polynomial filter families adapt standard \gls{GSO}-based filtering to subgraphs, but their performance is suboptimal because their filter spaces are inherently limited. For the same maximum order $r=3$, a polynomial filter captures only a single message-passing pattern, whereas the \gls{SFFA} filter space supports a richer set of patterns.

\item The proposed \gls{SFFA} filter space $\bbR_{\leq r}\langle\tbL_\Vsub^{[1]},\tbL_\Vsub^{[2]},\ldots,\tbL_\Vsub^{[k]}\rangle$ achieves the best performance. Under the same parameter budget, it consistently outperforms all polynomial filters and $\bbR_{\leq r}[\tbL_\Vsub^{[1]},\tbL_\Vsub^{[2]},\ldots,\tbL_\Vsub^{[k]}]$ by leveraging a richer noncommutative algebra generated by multiple subgraph-supported primitives. Moreover, by using distance-based subgraph Laplacians as primitives, this filter space encodes more informative message-passing schemes than $\bbR_{\leq r}\langle\tbL_{\mathrm{rand},1},\tbL_{\mathrm{rand},2},\ldots,\tbL_{\mathrm{rand},k}\rangle$, which is built from randomly selected sub-subgraphs.

\end{itemize}

We also conduct significance tests for each fixed subgraph ratio $\gamma$. We use paired one-sided $t$-tests with the proposed \gls{SFFA} as the reference method. Specifically, for each competing method $\frakM$, we form the paired loss differences
\begin{align*}
d_i\tc{\frakM,p}={\calL_{\frakM,i}}\tc{p}-{\calL_{\mathrm{SFFA},i}}\tc{p},\qquad i=1,\ldots,10,
\end{align*}
where a positive mean difference indicates that \gls{SFFA} achieves a smaller test loss. The null and alternative hypotheses are therefore
\begin{align*}
H_0:\mathbb{E}[{d_i}\tc{\frakM,p}]\leq 0,\qquad
H_1:\mathbb{E}[{d_i}\tc{\frakM,p}]>0.
\end{align*}
For each subgraph ratio $p$, the resulting $p$-values from comparing \gls{SFFA} against all competing methods are corrected using the Holm-Bonferroni procedure with a predetermined significance level $\alpha_s=0.05$.

According to \cref{tab:exp_sfl_filtering_pval}, the null hypothesis is rejected for all comparisons under each fixed subgraph ratio $\gamma$, supporting the statistical superiority of \gls{SFFA} for $\gamma=0.1,0.2,\ldots,0.9$.

\subsection{Subgraph Signal Reconstruction}\label{subsec:exp-reconstruction}

We evaluate multiple subgraph filtering methods on a bandlimited signal reconstruction task using subgraphs derived from the Windmill large dataset \cite{pygtemporal}. 
To construct the ambient graph $G$ with $N=319$ nodes, we first build an unweighted, symmetric $5$-nearest-neighbor graph, which yields two connected components. We then connect these components by adding a single inter-component edge corresponding to the largest original proximity weight between them, resulting in the final graph $G$ with a total of $999$ edges. The raw Windmill dataset contains $T=17,472$ timestamps, each associated with a graph signal defined on the node set $V$ of $G$. From these raw signals, we construct the target signal $\by^i$, $i=1,\dots,T$, as a $K$-bandlimited signal with respect to the normalized Laplacian $\tbL$ of $G$, with $K=250$. The input signal $\bx^i$ is obtained by masking $\by^i$ over a randomly selected missing-node set $\Vsub_M\subset\Vsub$ of size $N_M$. We use the first $n=200$ raw signals for training and the subsequent $n_0=200$ raw signals for testing.

\subsubsection{Different subgraph sizes}\label{subsubsec:p-sweep}

For each trial of \gls{SFL} on reconstruction with $(\bx,\by)$ restricted to a subgraph, we randomly pick a subset $\Vsub\subset V$ with $N_0=\gamma N$ nodes to generate the subgraph $\Gsub$, with $\gamma\in\{0.3,0.4,0.5,0.6,0.7\}$. We further randomly choose a missing node set $\Vsub_M\subset\Vsub$ with $N_M=\gamma_1N$ nodes, where $p_1=0.1$, and the goal is to reconstruct $\by_\Vsub$ from $\bx_\Vsub$.

Besides numerical LMMSE reference, we also report an asymptotic oracle LMMSE baseline, which estimates the unrestricted subgraph LMMSE operator using the subgraph signals from all $T=17,472$ timestamps, thereby approximating the oracle LMMSE filter.

We compare the proposed distance-based \gls{SFFA} with other structured \gls{SFL} filter banks and the numerical LMMSE baseline. We set the parameter budget $\alpha N_0^\beta$ with $\alpha=\beta=1$ according to \cref{eq:budget}. We fix $r=2$ and maximize $k$ to exploit the budget, leading to $k=9,10,12,13,14$ for $\gamma=0.3,0.4,0.5,0.6,0.7$, respectively. For the three single-\gls{GSO} polynomial baselines, we separately increase their degrees from $2$ to $6$ and report the best result within each family to indicate the upper limit of these polynomial filters.

\begin{table*}[!t]
\centering
    \caption{Test RMSE $\times10^{-2}$ for subgraph signal reconstruction with parameter budget $1\cdot N_0^{1}$, with the \first{best}, \second{second}, and \third{third} best performance highlighted. LMMSE results are provided as references.}
\label{tab:exp_sfl_reconstruction_psweep}

\setlength{\tabcolsep}{3pt}
\resizebox{0.7\textwidth}{!}{%
\begin{tabular}{lccccc}
\toprule
Subgraph ratio
& $\gamma=0.3$
& $\gamma=0.4$
& $\gamma=0.5$
& $\gamma=0.6$
& $\gamma=0.7$ \\
$k$
& $9$
& $10$
& $12$
& $13$
& $14$ \\
\midrule
Asymptotic Oracle LMMSE
& $6.50\pm0.47$
& $5.76\pm0.33$
& $5.15\pm0.20$
& $4.49\pm0.34$
& $3.96\pm0.23$ \\
Numerical LMMSE
& $9.30\pm0.72$
& $9.85\pm0.82$
& $10.71\pm0.46$
& $12.48\pm1.01$
& $22.43\pm3.07$ \\
\midrule
Best $\bbR_{\leq r}[\tbL_{\mathrm{ind}}]$
& $16.81\pm2.00$
& $14.62\pm1.26$
& $12.92\pm2.40$
& $11.40\pm1.10$
& $11.59\pm1.88$ \\
Best $\bbR_{\leq r}[\tbL_{\mathrm{Kron}}]$
& $\third{11.97\pm1.59}$
& $11.30\pm0.87$
& $11.52\pm1.46$
& $10.76\pm1.00$
& $11.03\pm1.39$ \\
Best $\bbR_{\leq r}[\tbL_{\Vsub}^{\cup_k}]$
& $12.33\pm0.84$
& $11.99\pm0.64$
& $12.44\pm0.90$
& $12.51\pm0.70$
& $12.47\pm0.90$ \\
\midrule
$\bbR_{\leq 2}[\tbL_{\Vsub}^{[1]},\ldots,\tbL_{\Vsub}^{[k]}]$
& $\second{10.92\pm1.30}$
& $\second{9.97\pm0.72}$
& $\third{9.94\pm0.97}$
& $\third{9.39\pm0.54}$
& $\third{9.27\pm0.83}$ \\
$\bbR_{\leq 2}\langle\tbL_{\mathrm{rand},1},\ldots,\tbL_{\mathrm{rand},k}\rangle$
& $14.57\pm2.12$
& $\third{11.27\pm1.72}$
& $\second{8.98\pm2.18}$
& $\first{7.37\pm0.46}$
& $\second{7.74\pm1.57}$ \\
\midrule
$\bbR_{\leq 2}\langle\tbL_{\Vsub}^{[1]},\ldots,\tbL_{\Vsub}^{[k]}\rangle$
& $\first{8.86\pm0.62}$
& $\first{8.37\pm0.55}$
& $\first{8.13\pm0.61}$
& $\second{7.56\pm0.39}$
& $\first{7.44\pm0.62}$ \\
\bottomrule
\end{tabular}%
}
\vspace{-3mm}
\end{table*}

\begin{table*}[!t]
\centering
\caption{Holm-Bonferroni-adjusted $p$-values of significance test over proposed \gls{SFFA} in subgraph signal reconstruction task. Non-significant results at level $\alpha_s=0.05$ are boldfaced.}
\label{tab:exp_sfl_reconstruction_pval}
\setlength{\tabcolsep}{3pt}

\resizebox{0.7\textwidth}{!}{%
\begin{tabular}{lccccc}
\toprule
Subgraph ratio
& $\gamma=0.3$
& $\gamma=0.4$
& $\gamma=0.5$
& $\gamma=0.6$
& $\gamma=0.7$ \\
$k$
& $9$
& $10$
& $12$
& $13$
& $14$ \\
\midrule
Best $\bbR_{\leq r}[\tbL_{\mathrm{ind}}]$
& $1.30{\times}10^{-7}$
& $1.64{\times}10^{-7}$
& $4.02{\times}10^{-5}$
& $5.81{\times}10^{-7}$
& $3.91{\times}10^{-6}$ \\
Best $\bbR_{\leq r}[\tbL_{\mathrm{Kron}}]$
& $9.49{\times}10^{-6}$
& $3.16{\times}10^{-7}$
& $6.07{\times}10^{-6}$
& $9.30{\times}10^{-7}$
& $3.27{\times}10^{-7}$ \\
Best $\bbR_{\leq r}[\tbL_{\Vsub}^{\cup_k}]$
& $1.24{\times}10^{-8}$
& $8.29{\times}10^{-9}$
& $2.66{\times}10^{-9}$
& $4.59{\times}10^{-9}$
& $5.83{\times}10^{-10}$ \\
\midrule
$\bbR_{\leq 2}[\tbL_{\Vsub}^{[1]},\ldots,\tbL_{\Vsub}^{[k]}]$
& $1.15{\times}10^{-5}$
& $7.81{\times}10^{-7}$
& $2.79{\times}10^{-6}$
& $4.48{\times}10^{-7}$
& $4.58{\times}10^{-7}$ \\
$\bbR_{\leq 2}\langle\tbL_{\mathrm{rand},1},\ldots,\tbL_{\mathrm{rand},k}\rangle$
& $4.20{\times}10^{-6}$
& $3.17{\times}10^{-4}$
& $\mathbf{1.00{\times}10^{-1}}$
& $\mathbf{9.50{\times}10^{-1}}$
& $\mathbf{2.04{\times}10^{-1}}$ \\
\bottomrule
\end{tabular}%
}
\vspace{-3mm}
\end{table*}

The results in \cref{tab:exp_sfl_reconstruction_psweep} indicate that the \gls{SFFA} filter bank with distance-based subgraph Laplacians consistently outperforms all competing methods for $\gamma=0.3,0.4,0.5$. As expected, as $\gamma$ and correspondingly $k$ increases, the performance becomes similar to random-group \gls{SFFA} since the random group now has a high probability of selecting \glspl{DSSO} with similar properties as the primitives in the distance-based subgraph Laplacian \gls{SFFA}. With only $200$ training samples, numerical LMMSE performs suboptimal because of the numerical instability issues; nevertheless, since it uses the full $N_0^2$-dimensional matrix space, it is reported only as a reference and excluded from the competing methods.

These results demonstrate that with appropriately chosen Laplacian primitives, a maximum order of $r=2$ suffices to construct a sufficiently expressive \gls{SFFA} filter space. In contrast, traditional \gls{GSO} filtering quickly reaches its representational limits regardless of the maximum order $r$, as it is inherently constrained to a single message propagation scheme. For reference, the asymptotic oracle LMMSE provides an empirical approximation to the population lower bound for subgraph-only linear filters under the $L^2$ loss.

We also conduct the same significance test as in \cref{subsec:exp-sfl}, for each fixed $\gamma$ and significance level $\alpha_s=0.05$\cite{TraSme:J12}. Referring to \cref{tab:exp_sfl_reconstruction_pval}, excluding LMMSE-based methods, the corrected tests show that our proposed \gls{SFFA} is significantly better than all competing methods, except for $\bbR_{\leq 2}\langle\tbL_{\mathrm{rand},1},\tbL_{\mathrm{rand},2},\ldots,\tbL_{\mathrm{rand},k}\rangle$ at $\gamma=0.5,0.6,0.7$, where the differences are not statistically significant after correction.

\subsubsection{Parameter budget analysis}\label{subsubsec:param-sweep}

\begin{figure}[!htb]
\begin{center}
\includegraphics[width=\columnwidth]{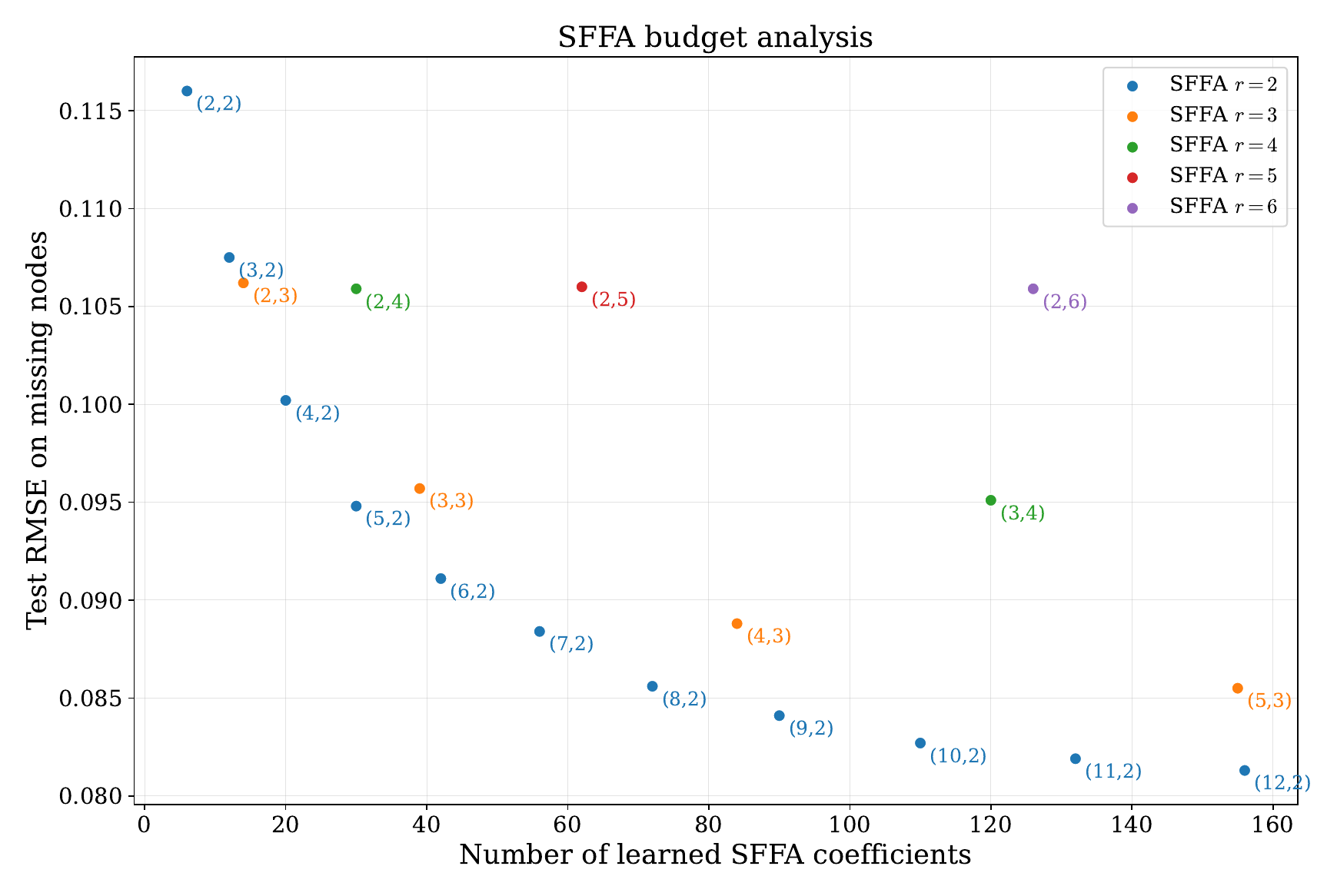}
\end{center}
\vspace{-5mm}
\caption{Test RMSE for subgraph signal reconstruction for different parameters $(k,r)$.}
\vspace*{-3pt}
\label{fig:exp_rec_paramsweep}
\end{figure}

We fix $\gamma=0.5$ and report the performance of the distance-based subgraph Laplacian \gls{SFFA} over all $(k,r)$ pairs with $k,r\geq 2$, subject to the constraint that the number of trainable parameters does not exceed the budget $N_0$. For reference, we also include the performance of numerical LMMSE, as well as the best-performing instances of $\bbR_{\leq r}[\tbL_{\mathrm{ind}}]$, $\bbR_{\leq r}[\tbL_{\mathrm{Kron}}]$, and $\bbR_{\leq r}[\tbL_\Vsub^{\cup_{12}}]$. 
The following observations can be drawn from \cref{fig:exp_rec_paramsweep}:
\begin{itemize}
    \item The polynomial filter spaces induced by single-\gls{GSO} filtering on the subgraph are inherently limited in expressiveness, causing all three single-\gls{GSO} methods to perform worse than numerical LMMSE. Importantly, these structural limitations cannot be overcome by increasing the maximum filter order.
    \item Nearly all tested $(k,r)$ pairs for the distance-based subgraph Laplacian \gls{SFFA} outperform numerical LMMSE, with the sole exception of $(2,2)$, further underscoring the importance of a sufficiently expressive subgraph filter bank.
    \item The graph $G$ in our experimental setting ($N=319$ nodes) is sparse with 999 edges. As discussed in \cref{subsec:expressiveness}, such sparsity favors \gls{SFFA} designs with more primitives over those with higher maximum order. This is reflected in the results in two ways: fixing $k$ and increasing $r$ yields no notable improvement for $r\geq 4$, whereas fixing $r$ and increasing $k$ consistently improves the performance of the subgraph filter bank.
\end{itemize}

\subsection{Traffic Flow Prediction on Subgraphs}\label{subsec:exp-traffic-pred}

We now evaluate the performance of different \gls{SFL} methods on a subgraph traffic flow prediction task using the same dataset METR-LA dataset \cite{li18} in \cref{subsec:exp-sfl}. We split the data in a 7:1:2 ratio for training, validation, and testing, respectively.

Given input and output horizons $T_{\mathrm{in}}$ and $T_{\mathrm{out}}$, each input-output subgraph signal pair indexed by $(t_1,t_2)$, where $1\leq t_1\leq T_{\mathrm{in}}$ and $1\leq t_2\leq T_{\mathrm{out}}$, is modeled by a linear subgraph filter $\Fsub^{t_1,t_2}\in\spn\calS(\Gsub)$. Specifically, for each training window, the prediction is given by
\begin{align*}
\Psub\hat{\by}^{t_2}=\sum_{t_1=1}^{T_{\mathrm{in}}}\Fsub\tc{t_1,t_2}\Psub\bx^{t_1}.
\end{align*}
Therefore, the traffic prediction task can be formulated as a multi-input multi-output \gls{SFL} problem in which all $T_{\mathrm{in}}\times T_{\mathrm{out}}$ filters belong to the same subgraph filter space $\spn\calS(\Gsub)$ but use distinct coefficients.

For prediction, we set the input horizon to $T_{\mathrm{in}}=12$ and the output horizon to $T_{\mathrm{out}}=3$, and slide time windows of length $T_{\mathrm{in}}+T_{\mathrm{out}}$ across each split of the dataset.
For each \gls{SFL} trial on subgraph prediction, we select a subset $\Vsub\subset V$ with $N_0=\gamma N$ nodes to construct the subgraph $\Gsub$, where $\gamma=0.3,0.4,\ldots,0.7$.
For single-\gls{GSO} filter banks, we set $r=3$. For $\bbR_{\leq r}[\bL_{\Vsub}^{[1]},\bL_{\Vsub}^{[2]},\ldots,\bL_{\Vsub}^{[k]}]$, $\bbR_{\leq r}\langle\bL_{\mathrm{rand},1},\bL_{\mathrm{rand},2},\ldots,\bL_{\mathrm{rand},k}\rangle$, and $\bbR_{\leq  r}\langle\bL_{\Vsub}^{[1]},\bL_{\Vsub}^{[2]},\ldots,\bL_{\Vsub}^{[k]}\rangle$, we use $(k,r)=(10,2)$. We optimize using the least-squares loss.

In this task, we use the \gls{AR} model as the basic baseline:
\begin{align*}
\hat y^{t_2}_u=\sum_{t_1=1}^{T_{\mathrm{in}}}a\tc{t_1,t_2}x^{t_1}_u,
\end{align*}
where $x^{t_1}_u$ denotes the scalar input signal at node $u\in\Vsub$, extracted from $\bx^{t_1}$. This linear model treats the time series at each node independently.
For each method and each $p$, we report the test MSE and the \gls{RI} over AR. Additional optimization details are provided in \cref{sec:implementation} in the supplementary material.

\begin{table*}[!t]
\centering
\caption{Average test MSE (first sub-row) and \gls{RI} (second sub-row) with standard deviation for traffic flow prediction on METR-LA, with the \first{best}, \second{second}, and \third{third} best performance highlighted. }
\label{tab:exp_traffic_psweep}
\setlength{\tabcolsep}{3pt}

\resizebox{0.7\textwidth}{!}{%
\begin{tabular}{lccccc}
\toprule
Subgraph ratio 
& $\gamma=0.3$
& $\gamma=0.4$
& $\gamma=0.5$
& $\gamma=0.6$
& $\gamma=0.7$ \\
\midrule
AR
& $69.09\pm 1.75$
& $69.44\pm 1.16$
& $69.81\pm 0.75$
& $69.88\pm 1.01$
& $69.89\pm 1.01$ \\
\midrule

\multirow{2}{*}{$\bbR_{\leq 3}[\tbL_{\mathrm{ind}}]$}
& $68.79\pm 1.70$
& $69.04\pm 1.15$
& $69.36\pm 0.73$
& $69.37\pm 0.97$
& $69.32\pm 0.95$ \\
& $0.44\pm 0.12\%$
& $0.57\pm 0.10\%$
& $0.64\pm 0.08\%$
& $0.74\pm 0.06\%$
& $0.81\pm 0.09\%$ \\
\cmidrule(lr){2-6}

\multirow{2}{*}{$\bbR_{\leq 3}[\tbL_{\mathrm{Kron}}]$}
& $68.76\pm 1.73$
& $69.01\pm 1.17$
& $69.35\pm 0.75$
& $69.36\pm 0.95$
& $69.33\pm 0.94$ \\
& $0.48\pm 0.07\%$
& $0.61\pm 0.08\%$
& $0.65\pm 0.08\%$
& $0.74\pm 0.09\%$
& $0.80\pm 0.10\%$ \\
\cmidrule(lr){2-6}

\multirow{2}{*}{$\bbR_{\leq 3}[\tbL_\Vsub^\cup]$}
& $68.89\pm 1.74$
& $69.23\pm 1.18$
& $69.58\pm 0.74$
& $69.65\pm 0.98$
& $69.64\pm 0.97$ \\
& $0.29\pm 0.07\%$
& $0.30\pm 0.09\%$
& $0.32\pm 0.08\%$
& $0.33\pm 0.05\%$
& $0.35\pm 0.07\%$ \\
\midrule

\multirow{2}{*}{$\bbR_{\leq 2}[\tbL_\Vsub^{[1]},\ldots,\tbL_\Vsub^{[10]}]$}
& $\third{68.65\pm 1.74}$
& $\third{68.92\pm 1.18}$
& $\third{69.23\pm 0.74}$
& $\third{69.27\pm 0.95}$
& $\third{69.23\pm 0.93}$ \\
& $\third{0.63\pm 0.12\%}$
& $\third{0.74\pm 0.10\%}$
& $\third{0.82\pm 0.09\%}$
& $\third{0.88\pm 0.10\%}$
& $\third{0.93\pm 0.11\%}$ \\
\cmidrule(lr){2-6}

\multirow{2}{*}{$\bbR_{\leq 2}\langle\tbL_{\mathrm{rand},1},\ldots,\tbL_{\mathrm{rand},10}\rangle$}
& $\second{68.44\pm 1.65}$
& $\second{68.76\pm 1.15}$
& $\second{69.11\pm 0.71}$
& $\second{69.15\pm 0.99}$
& $\second{69.12\pm 0.95}$ \\
& $\second{0.94\pm 0.22\%}$
& $\second{0.98\pm 0.11\%}$
& $\second{1.00\pm 0.10\%}$
& $\second{1.05\pm 0.09\%}$
& $\second{1.09\pm 0.08\%}$ \\
\midrule

\multirow{2}{*}{$\bbR_{\leq 2}\langle\tbL_\Vsub^{[1]},\ldots,\tbL_\Vsub^{[10]}\rangle$}
& $\first{68.40\pm 1.74}$
& $\first{68.62\pm 1.17}$
& $\first{68.94\pm 0.76}$
& $\first{68.99\pm 0.95}$
& $\first{68.97\pm 0.92}$ \\
& $\first{1.00\pm 0.26\%}$
& $\first{1.18\pm 0.16\%}$
& $\first{1.24\pm 0.07\%}$
& $\first{1.28\pm 0.11\%}$
& $\first{1.31\pm 0.14\%}$ \\
\bottomrule
\end{tabular}%
}
\vspace{-3mm}

\end{table*}

\begin{table*}[!t]
\centering
\caption{Holm-Bonferroni-adjusted $p$-values of significance test over proposed \gls{SFFA} in subgraph signal prediction task. Non-significant results at level $\alpha_s=0.05$ are boldfaced.}
\label{tab:exp_traffic_pval}
\setlength{\tabcolsep}{3pt}

\resizebox{0.7\textwidth}{!}{%
\begin{tabular}{lccccc}
\toprule
Subgraph ratio 
& $\gamma=0.3$
& $\gamma=0.4$
& $\gamma=0.5$
& $\gamma=0.6$
& $\gamma=0.7$ \\
\midrule
AR
& $2.42\times 10^{-6}$
& $5.81\times 10^{-9}$
& $3.07\times 10^{-12}$
& $2.25\times 10^{-10}$
& $1.53\times 10^{-9}$ \\
\midrule

$\bbR_{\leq 3}[\tbL_{\mathrm{ind}}]$
& $1.09\times 10^{-4}$
& $4.26\times 10^{-6}$
& $1.34\times 10^{-8}$
& $4.26\times 10^{-9}$
& $6.43\times 10^{-9}$ \\

$\bbR_{\leq 3}[\tbL_{\mathrm{Kron}}]$
& $1.09\times 10^{-4}$
& $2.05\times 10^{-6}$
& $1.03\times 10^{-8}$
& $3.87\times 10^{-9}$
& $4.88\times 10^{-9}$ \\

$\bbR_{\leq 3}[\tbL_\Vsub^\cup]$
& $1.90\times 10^{-5}$
& $9.02\times 10^{-9}$
& $3.99\times 10^{-11}$
& $2.25\times 10^{-10}$
& $1.29\times 10^{-9}$ \\
\midrule

$\bbR_{\leq 2}[\tbL_\Vsub^{[1]},\ldots,\tbL_\Vsub^{[10]}]$
& $1.23\times 10^{-4}$
& $2.74\times 10^{-7}$
& $1.48\times 10^{-8}$
& $3.87\times 10^{-9}$
& $6.20\times 10^{-9}$ \\

$\bbR_{\leq 2}\langle\tbL_{\mathrm{rand},1},\ldots,\tbL_{\mathrm{rand},10}\rangle$
& $\bm{2.75\times 10^{-1}}$
& $3.12\times 10^{-4}$
& $3.45\times 10^{-5}$
& $1.23\times 10^{-4}$
& $8.55\times 10^{-6}$ \\
\bottomrule
\end{tabular}%
}
\vspace{-3mm}
\end{table*}

From \cref{tab:exp_traffic_psweep}, we observe that the proposed distance-based subgraph Laplacian \gls{SFFA} consistently outperforms competing methods with stable \gls{RI} values. These findings demonstrate the advantage of an \gls{SFFA}-based subgraph filter bank in conjunction with structured \gls{DSSO} primitives, even for \gls{SFL} tasks with multiple inputs and outputs.

We also conduct the same significance test as in \cref{subsec:exp-sfl}, for each fixed $\gamma$ and significance level $\alpha_s=0.05$. Referring to \cref{tab:exp_traffic_pval}, the corrected tests show that our proposed \gls{SFFA} is significantly better than all competing methods, except for $\bbR_{\leq 2}\langle\tbL_{\mathrm{rand},1},\tbL_{\mathrm{rand},2},\ldots,\tbL_{\mathrm{rand},k}\rangle$ at $\gamma=0.3$.

\section{Conclusion}\label{sec:conclusion}

In this paper, we introduced a framework for subgraph filter learning and developed a principled method based on subgraph filter free algebras for learning under partial graph observations. By leveraging distance-$k$ subgraphs, the proposed construction captures multi-scale neighborhood interactions in the ambient graph, yielding a structured Laplacian algebra that is both interpretable and expressive. From a theoretical standpoint, we characterized the performance gap between learned subgraph filters and the ground-truth mapping, and showed that this gap arises from multiple sources, including the intrinsic limitations of linear filtering, distance-based restrictions of the filter space, numerical approximation error, and regularization effects. These findings clarify key trade-offs in subgraph-based learning. Empirically, experiments on real-world traffic and wind-speed datasets show that the proposed framework consistently outperforms numerical filter-approximation methods and polynomial subgraph filters, underscoring its practical effectiveness.

Future work may consider more challenging settings, including distribution shift, time-varying signals, and signals on higher-order structures. One promising direction is joint spatio-temporal \gls{SFL} \cite{ZhaTay:C26SAM}, which extends the present framework to discrete-time graph signals by learning subgraph-supported joint-domain filters via a spatio-temporal filter algebra that couples spatial filtering with temporal dependency modeling. Another direction is subgraph sheaf diffusion \cite{ZhaTay:C26SSP}, which generalizes the structural component of \gls{SFL} from subgraph Laplacian-based filters to sub-sheaf diffusion operators, where local spaces and restriction maps provide a more flexible mechanism for aligning heterogeneous or phase-shifted signals.


\appendices

\section{Selected Proofs}\label[Appendix]{sec:selected-proofs}
In this section we provide the proof sketch for \cref{thm:hp-ridge-decomp} and the complete proof of \cref{thm:decomposition-bounds}, \cref{thm:hp-ridge-oracle}. Other proofs are provided in the supplementary material.

\begin{proof}[Proof sketch of \cref{thm:hp-ridge-decomp}.] 

The decomposition follows by tracking which splits are orthogonal. 

First, since $\bw$ is uncorrelated with both $\bx$ and $f(\bx)$, the cross terms vanish, and hence
\begin{align*}
    \E\norm{\Psub f(\bx)-\hat{\bF}_{m,\eta}\tc{n}\bx_\Vsub}_2^2
    \leq\calR_{L^2}(\hat{\bF}_{m,\eta}\tc{n}).
\end{align*}

Then, we split the risk through $\Fsub^*$ and $\bar{\bF}_m$:
\begin{align*}
    \calR_{L^2}(\hat{\bF}_{m,\eta}\tc{n})
    =&\calR_{L^2}(\Fsub^*)
    +\parens*{\calR_{L^2}(\bar{\bF}_m)-\calR_{L^2}(\Fsub^*)}\\
    &+\parens*{\calR_{L^2}(\hat{\bF}_{m,\eta}\tc{n})-\calR_{L^2}(\bar{\bF}_m)}.
\end{align*}
The first difference is an equality because the residual of $\Fsub^*$ is orthogonal to every direction in $\bbR^{N_0\times N_0}$, so
\begin{align*}
    \calR_{L^2}(\bar{\bF}_m)-\calR_{L^2}(\Fsub^*)
    =
    \norm{\bar{\bF}_m-\Fsub^*}_{\Msub}^2.
\end{align*}
The second difference is also an equality because the residual of $\bar{\bF}_m$ is orthogonal to every direction in $\calV_k$, and
$\hat{\bF}_{m,\eta}\tc{n}-\bar{\bF}_m\in\calV_k$. Hence
\begin{align*}
    \calR_{L^2}(\hat{\bF}_{m,\eta}\tc{n})-\calR_{L^2}(\bar{\bF}_m)
    =
    \norm{\hat{\bF}_{m,\eta}\tc{n}-\bar{\bF}_m}_{\Msub}^2.
\end{align*}

Finally, we split the last norm:
\begin{align*}
    \hat{\bF}_{m,\eta}\tc{n}-\bar{\bF}_m
    =
    \big(\hat{\bF}_{m,\eta}\tc{n}-\bar{\bF}_{m,\eta}\big)
    +
    \big(\bar{\bF}_{m,\eta}-\bar{\bF}_m\big).
\end{align*}
This split is not orthogonal, so it only gives an inequality with a factor $2$:
\begin{align*}
    \norm{\hat{\bF}_{m,\eta}\tc{n}-\bar{\bF}_m}_{\Msub}^2
    \leq
    2\norm{\bar{\bF}_{m,\eta}-\hat{\bF}_{m,\eta}\tc{n}}_{\Msub}^2\\
    + 2\norm{\bar{\bF}_m-\bar{\bF}_{m,\eta}}_{\Msub}^2.
\end{align*}
Combining these three splits gives \cref{eq:hp-ridge-full}.
\end{proof}

\begin{proof}[Proof of \cref{thm:hp-ridge-oracle}]

Let $\bepsilon(\bx)\coloneq f(\bx)-\bJ\bx$. Then, 
\begin{align*}
\E[\bepsilon(\bx)\bx\T]=\E[f(\bx)\bx\T]-\bJ\E[\bx\bx\T]\\
=\E[f(\bx)\bx\T]-\E[f(\bx)\bx\T]\bM^{-1}\bM=\bzero.
\end{align*}
Partition $\bJ$ and $\bM$ according to $\Vsub$ and $\Vsubc$. Then
\begin{align*}
\Psub f(\bx)
=\bJ_{\Vsub\Vsub}\bx_{\Vsub}+\bJ_{\Vsub\Vsubc}\bx_{\Vsubc}+\Psub\bepsilon(\bx).
\end{align*}
By the definition of the oracle LMMSE subgraph filter and by the assumption that the noise is uncorrelated with $\bx$, we have $\Fsub^*=\E[\by_{\Vsub}\bx_{\Vsub}\T]\Msub^{-1}=\E[\Psub f(\bx)\bx_{\Vsub}\T]\Msub^{-1}$. Using the decomposition above and $\E[\bepsilon(\bx)\bx_{\Vsub}\T]=\bzero$, we obtain $\E[\Psub f(\bx)\bx_{\Vsub}\T]=\bJ_{\Vsub\Vsub}\Msub+\bJ_{\Vsub\Vsubc}\bM_{\Vsubc\Vsub}$. Hence, $\Fsub^*=\bJ_{\Vsub\Vsub}+\bJ_{\Vsub\Vsubc}\bM_{\Vsubc\Vsub}\Msub^{-1}$. Therefore,
\begin{align*}
\Psub f(\bx)-\Fsub^*\bx_{\Vsub}
&=\bJ_{\Vsub\Vsubc}\left(\bx_{\Vsubc}-\bM_{\Vsubc\Vsub}\Msub^{-1}\bx_{\Vsub}\right)+\Psub\bepsilon(\bx).
\end{align*}

For brevity, define $\br_{\Vsubc|\Vsub}\coloneq\bx_{\Vsubc}-\bM_{\Vsubc\Vsub}\Msub^{-1}\bx_{\Vsub}.$ Then, 
\begin{align*} 
\E[\br_{\Vsubc|\Vsub}\br_{\Vsubc|\Vsub}\T]&=\E\bigg[\bigg(\bx_{\Vsubc}-\bM_{\Vsubc\Vsub}\Msub^{-1}\bx_{\Vsub}\bigg)\\
&\qquad\quad \bigg(\bx_{\Vsubc}-\bM_{\Vsubc\Vsub}\Msub^{-1}\bx_{\Vsub}\bigg)\T \bigg]\\
&=\bM_{\Vsubc\Vsubc}-\bM_{\Vsubc\Vsub}\Msub^{-1}\bM_{\Vsub\Vsubc}. 
\end{align*}

Moreover, since $\E[\bepsilon(\bx)\bx\T]=\bzero$, we also have
\begin{align*}
\E[\Psub\bepsilon(\bx)\br_{\Vsubc|\Vsub}\T]
&=\E[\Psub\bepsilon(\bx)\bx_{\Vsubc}\T]\\
&\quad-\E[\Psub\bepsilon(\bx)\bx_{\Vsub}\T]\Msub^{-1}\bM_{\Vsub\Vsubc}\\
&=\bzero.
\end{align*}

Thus the cross term vanishes, and
\begin{align*}
\calR_{L^2}(\Fsub^*)
&=\E\norm{\Psub f(\bx)-\Fsub^*\bx_{\Vsub}}_2^2 \\
&=\E\norm{\bJ_{\Vsub\Vsubc}\br_{\Vsubc|\Vsub}+\Psub\bepsilon(\bx)}_2^2 \\
&=\tr\left(\bJ_{\Vsub\Vsubc}\left(\bM_{\Vsubc\Vsubc}-\bM_{\Vsubc\Vsub}\Msub^{-1}\bM_{\Vsub\Vsubc}\right)\bJ_{\Vsub\Vsubc}\T\right)\\
&\qquad+\E\norm{\Psub\bepsilon(\bx)}_2^2.
\end{align*}

By the definition of $\sigma_{\Vsub}^2$, this becomes
\begin{align*}
&\calR_{L^2}(\Fsub^*)\\
&\quad=\sigma_{\Vsub}^2+\tr\left(\bJ_{\Vsub\Vsubc}\left(\bM_{\Vsubc\Vsubc}-\bM_{\Vsubc\Vsub}\Msub^{-1}\bM_{\Vsub\Vsubc}\right)\bJ_{\Vsub\Vsubc}\T\right).
\end{align*}

Using the standard trace bound
\begin{align*}
\tr(\bS\bC\bS\T)\leq\lambdamax{\bC}\norm{\bS}_F^2,\qquad \bC\succeq \bzero,
\end{align*}
with $\bS=\bJ_{\Vsub\Vsubc}$ and $\bC=\bM_{\Vsubc\Vsubc}-\bM_{\Vsubc\Vsub}\Msub^{-1}\bM_{\Vsub\Vsubc}$, and using the Loewner order of Schur complement, we get
\begin{align*}
&\calR_{L^2}(\Fsub^*)-\sigma_{\Vsub}^2\\
&\quad\leq\lambdamax{\bM_{\Vsubc\Vsubc}-\bM_{\Vsubc\Vsub}\Msub^{-1}\bM_{\Vsub\Vsubc}}\norm{\bJ_{\Vsub\Vsubc}}_F^2\\
&\quad\leq\lambdamax{\bM_{\Vsubc\Vsubc}}\norm{\bJ_{\Vsub\Vsubc}}_F^2\\
&\quad\leq\kappa_{\max}\norm{\bJ_{\Vsub\Vsubc}}_F^2.\\
\end{align*}
Finally, by the definition of $s_\partial$,
\begin{align*}
\norm{\bJ_{\Vsub\Vsubc}}_F^2=\sum_{u\in\Vsub}\sum_{v\in\Vsubc}\left|\left(\bJ\right)_{uv}\right|^2\leq|\Vsub||\Vsubc|s_\partial^2.
\end{align*}

Substituting this into the previous display gives
\begin{align*}
\calR_{L^2}(\Fsub^*)\leq\sigma_{\Vsub}^2+\kappa_{\max}|\Vsub||\Vsubc|s_\partial^2,
\end{align*}
which concludes the proof.

\end{proof}

\bibliographystyle{IEEEtran}
\bibliography{bib/IEEEabrv,bib/StringDefinitions,bib/SIGNAL,bib/mybib}

\end{document}